\documentclass[aps,pre,preprint,groupedaddress]{revtex4-2}

\usepackage[table]{xcolor}

\usepackage{graphicx}
\usepackage[linesnumbered, algo2e, vlined]{algorithm2e}
\usepackage{booktabs}
\usepackage{bbm}

\usepackage{tabularx}

\usepackage{todonotes}
\presetkeys{todonotes}{color=yellow, inline}{}

\usepackage{amsmath}
\usepackage{amsfonts}
\usepackage{amssymb}

\usepackage[T1]{fontenc}
\usepackage[caption=false]{subfig}
\usepackage{units}
\usepackage{floatrow}
\usepackage{bm}
\usepackage{textcomp}

\usepackage{url}

\newcommand{\bigohpar}[1]{\ensuremath{\mathcal{O} \left(#1\right)}}

\newcommand{\drifter}{\texttt{drifter}\xspace}

\newcommand{\airquality}{\textsc{aq}\xspace}
\newcommand{\bikesharing}{\textsc{bike}\xspace}
\newcommand{\bikesharingr}{\textsc{bike(raw)}\xspace}
\newcommand{\bikesharingd}{\textsc{bike(detr)}\xspace}
\newcommand{\airline}{\textsc{airline}\xspace}
\newcommand{\synthetic}{\textsc{synthetic}\xspace}

\newcommand{\syntheticns}{\textsc{synthetic}}

\usepackage{amsthm}
\newtheorem{theorem}{Theorem}
\newtheorem{problem}{Problem}

\begin{document}


\title{Estimating regression errors without ground truth values}


\author{Henri Tiittanen}
\author{Emilia Oikarinen}
\email[]{emilia.oikarinen@helsinki.fi}
\author{Andreas Henelius}
\author{Kai Puolam\"aki}
\email[]{kai.puolamaki@helsinki.fi}
\affiliation{Department of Computer Science, University of Helsinki, Helsinki, Finland}

\date{\today}

\begin{abstract}
Regression analysis is a standard supervised machine learning method
used to model an outcome variable in terms of a set of predictor
variables. In most real-world applications we do not know the true value
of the outcome variable being predicted outside the training data,
i.e., the ground truth is unknown. It is hence not
straightforward to directly observe
when the estimate from a model potentially is
wrong, due to phenomena such as overfitting and concept drift. In this
paper we present an efficient framework for estimating the
generalization error of regression functions, applicable to any family
of regression functions when the ground truth is unknown. We present
a theoretical derivation of the framework and 
empirically evaluate its strengths and limitations. We find that it
performs robustly and is useful for detecting concept drift in
datasets in several real-world domains.
\end{abstract}

\maketitle

\section{Introduction\label{sec:intro}}

Regression models are one of the most used and studied machine
learning primitives. For example, a bibliographic search of the
Physical Review E journal reveals that during 2018 the journal
published 11 articles containing the word ``regression'' already in
the title or abstract \cite{PhysRevE.98.043311, PhysRevE.98.043308,
  PhysRevE.98.043102, PhysRevE.98.032303, PhysRevE.98.022137,
  PhysRevE.98.022109, PhysRevE.98.012136, PhysRevE.97.063107,
  PhysRevE.97.062123, PhysRevE.97.022312, PhysRevE.97.012113}.  In
regression analysis, the idea is to estimate the value of the
dependent variable (denoted by $y\in{\mathbb{R}}$) given a
$m$-dimensional vector of covariates (here we assume real valued
attributes, denoted by $x\in{\mathbb{R}}^m$). The regression model is
trained using \emph{training data} in such a way that it gives good
estimates of the dependent variable on \emph{testing data} unseen
in the training phase. In addition to estimating the value of the
dependent variable, it is in practice important to know the
reliability of the estimate on testing data. In this paper, we use the
expected root mean square error (RMSE) between the dependent variable
and its estimate to quantify the uncertainty, but some other error
measure could be used as well.

In textbooks, one finds a plethora of ways to train various regression
models and to estimate uncertainties, see, e.g., \cite{Hastie2009}.
For example, for a Bayesian regression model the reliability of the
estimate can be expressed in terms of the posterior distribution or,
more simply, as a confidence interval around the estimate. Another
alternative to estimate the error of a regression estimate on yet
unseen data is to use (cross-)validation. 
All of these approaches give some measure of the
error on testing data, even when the dependent variable is unknown.

Textbook approaches are, however, \emph{valid only when the training
  and testing data obey the same distribution}. In many practical
applications this assumption does not hold:  an phenomenon known
as \emph{concept drift} \cite{Gama2016} occurs. Concept drift means that the
distribution of the data changes over time, in which case the
assumptions made by the regression model break down, resulting in
regression estimates with unknown and possibly large errors.
A typical example of concept drift occurs in sensor calibration, where
a regression model trained to model sensory response may fail when the
environmental conditions change from those used to train the
regression model \cite{Kadlec2011, Vergara2012, Rudnitskaya2018,
  Maag2018, huggard2018predicting}. Another example is given by online
streaming data applications such as sentiment classification
\cite{bifet2010sentiment} and e-mail spam detection
\cite{lindstrom2010handling}, where in order to optimize
the online training process, the model should only be retrained when
its performance is degraded, i.e., when concept drift is
detected.

At simplest, if the ground truth (the dependent variable $y$) is
known, concept drift may be detected simply by observing the magnitude
of the error, i.e., the difference between the regression estimate and
the dependent variable. However, in practice, this is often not the case. 
Indeed, usually the motive for using a regression model is
that the value of the dependent variable is not readily available. In
this paper, we address the problem of \emph{assessing the regression
  error when the ground truth is unknown}. It is surprising that
despite the significance of the problem it has not really been
adequately addressed in the literature, see Sec.~\ref{sec:related} for a discussion of related work. 
In this paper we do not
 focus on any particular domain, such as, e.g., sensor
calibration or sentiment or spam classification. Instead, our goal is
to introduce a generic computational methodology that can be applied
in a wide range of domains where regression is used.

Concept drift can be divided into two main categories, namely 
\emph{real concept drift} and \emph{virtual concept drift}. The
former refers to the change in the conditional probability $p(y\mid
x)$ and the latter to the change in the distribution of the covariates
$p(x)$, see, e.g., \cite[Sec 2.1]{Gama2016} for a discussion. If only
the covariates, i.e., $x$, are known but the ground truth, i.e., $y$,
is not, then it is not possible even in theory to detect changes
occurring in $p(y\mid x)$ only but not in $p(x)$. However, it is
possible to detect changes in $p(x)$ even when the values of $y$ 
 have not been observed. 
 For this reason, we focus on the
\emph{detection of virtual concept drift} in this paper. Note, that one
possible interpretation for a situation where $p(y\mid x)$ changes but
$p(x)$ remains unchanged is that we are missing some covariates from
$x$ which would parametrize the changes in $p(y\mid x)$. Therefore, an
occurrence of real concept drift without virtual concept drift can
indicate that we might not have all necessary attributes at our
disposal. An obvious solution is to include more attributes into the
set of covariates.

One should further observe, that when studying concept drift, we are not
interested in detecting merely any changes in the distribution of $x$.
Rather, we are only interested in changes that likely increase the
error of the regression estimates, a property which is satisfied by
our proposed method.

\subsection{Contributions and organization} In this paper we (i) define the
problem of detecting the concept drift which affects the regression
error when the ground truth is unknown, (ii) present an efficient
algorithm to solve the problem for arbitrary (black-box) regression
models, (iii) show theoretical properties of our solution, and (iv)
present an empirical evaluation of our approach.

The rest of this paper is structured as follows. In
Sec.~\ref{sec:related} we review the related work. In
Sec.~\ref{sec:methods} we introduce the idea behind our proposed
method for detecting virtual concept drift, which is then formalized
in the algorithm discussed in Sec.~\ref{ssec:drifter}. We demonstrate
different aspects of our method in the experimental evaluation in
Sec.~\ref{sec:experiments}. Finally, we conclude with a discussion in
Sec.~\ref{sec:discussion}.

\section{Related Work \label{sec:related}}

The term ``concept drift'' was coined by
 Schlimmer and Granger
\cite{schlimmer1986incremental} to describe the phenomenon where the
data distribution changes over time in dynamically changing and
non-stationary environments. The research related to concept drift has
become popular over the last decades with many real world
applications, see, e.g., the recent surveys \cite{Gama2016,
  zliobaite:etal2016, lu2018learning}.

Most of the concept drift literature focuses on classification
problems and concept drift adaptation problems.  In contrast,  
in this paper our
focus is on detecting virtual concept drift in
\emph{regression problems}. There are very few works on concept drift
in regression, although some of the ideas used with classifiers may be
applicable to regression functions.
Concept drift detection methods can be divided into
supervised (requiring ground truth values) and unsupervised (requiring
no ground truth values) approaches. Our approach falls into the latter
category, and in the following we focus on reviewing the unsupervised
approaches to concept drift detection.

We first briefly mention some {\em methods requiring ground truth
  values}. In \cite{gama:2004}, classifier prediction errors are used
to detect concept drift. The idea is to maintain concept windows, and
model prediction errors using a binomial distribution to obtain an
error probability for each window. Then, if the number of windows
containing errors is above a threshold, concept drift is detected.
The method proposed in \cite{wang:2017} is one of the few concept
drift detection methods for regression. There, an ensemble of multiple
regression models trained on sequences/subsets of the data is used to
find the best weighting for combining their predictions, and concept
drift is then defined as the angle between the estimated weight and
mean weight vectors. While the method in \cite{wang:2017} has
similarities to our method proposed here (i.e., training several
regressors on subsets of data), the fundamental difference is that
in \cite{wang:2017} the ground truth values are required.
Ikonomovska et al.~\cite{ikonomovska:2011} train models on subsets of the data and use
an ensemble of model trees where each model in the tree is trained on
different parts of the data. Concept drift is detected by monitoring
the model errors. 

When considering \emph{methods that require no ground truth values},
the approaches can be divided roughly into two categories: methods
detecting purely distributional changes and methods that also take
into account the model in some way. Concept drift detection approaches
based on directly monitoring the covariate distribution $p(x)$ detect
all changes in $p(x)$ regardless of their effect on the performance of
the prediction model, examples of such methods include, e.g.,
\cite{dasu2006information, DBLP:conf/kdd/ShaoAK14,
  DBLP:conf/kdd/QahtanAWZ15}. 
There are also approaches for covariate change detection without comparing distributions directly.
For instance, \cite{demello:2019} proposes a drift detection method using different
measure functions (e.g., statistical moments and power spectrum) on
particular time-series windows, and then a divergence value is used to classify windows to concept drift or not.
  However, if the task is to detect
concept drift that degrades the performance of the model, these
approaches suffer from a high false alarm rate \cite{sethi:2017}.

 The MD3 method  \cite{sethi:2017} 
uses classifier margin densities for concept drift detection, hence, requiring a
classifier that has some meaningful notion of margin, such as, e.g., 
a probabilistic classifier and a support
vector machine. The method works by dividing the input data into
segments, and for each segment the proportion of samples in the
margin, denoted by $\rho$, is computed. The minimum and
maximum values of $\rho$ are monitored, and if their difference
exceeds a given threshold, concept drift is declared.

In \cite{DBLP:journals/evs/LindstromND13},  a stream of indicator values  correlated to concept drift  is calculated
from test data
windows. If a certain proportion
of previous indicator values are above a threshold, concept drift is
declared. The indicator values are computed using the Kullback-Leibler
divergence to compare the histogram of classifier output confidence
scores on a test window to a reference window.
The method is not generic, however, since 
it requires a classifier producing a score that can be interpreted
as an estimate of the confidence associated with the correctness of
the prediction. 
Since probabilistic regression models provide direct information of
the model behavior in the form of uncertainty estimates, it is
straightforward to implement a concept drift detection measure by
thresholding the uncertainty estimate, e.g., 
in \cite{DBLP:conf/sdm/ChandolaV11} a method
based on Gaussian processes for time series change detection is presented. 

Further approaches to concept drift detection include 
 \cite{DBLP:journals/jucs/SobolewskiW13}, in which the method is
developed especially for data containing recurring concepts. Hence, the method in  \cite{DBLP:journals/jucs/SobolewskiW13}
requires prior knowledge about properties of concepts present in the
data, namely the samples residing in
the centers or at the borders of the class clusters,  to be incorporated into the model.
Then, a distinct classification model is trained for each concept, and for each test
data segment the closest concept in training data is selected using a
non-parametric statistical test. The test data segment is then
classified using that particular classification model. 
In \cite{DBLP:conf/icdm/Zliobaite10} concept drift detection for
binary classification is performed by comparing classifier output
label sequences to a reference window. It is assumed that the training
and testing data samples originate from two binomial distributions,
and concept drift is detected by using statistical testing with the
null hypothesis that the distributions are equal.

\section{Methods}
\label{sec:methods}
Let the \emph{training data} $D_\mathrm{tr}$ consist of
$n_\mathrm{tr}$ triplets: $D_\mathrm{tr}=\{(i, x_i, y_i)\}_{i=1}^{n_\mathrm{tr}}$, where $i \in
[n_\mathrm{tr}] = \{1, \ldots, n_\mathrm{tr}\}$ is the time index,
$x_i \in \mathbb{R}^m$ are the covariates and $y_i \in \mathbb{R}$ is
the dependent variable. Also, let the \emph{testing data} similarly be
given by $D_\mathrm{te} = \{(i, x'_i, y'_i)\}_{i=1}^{n_\mathrm{te}}$ where $i \in
[n_\mathrm{te}]$, and the covariates and the dependent variable are
given by $x'_i \in \mathbb{R}^m$ and $y'_i \in \mathbb{R}$,
respectively. Furthermore, let the \emph{reduced testing data} be the
testing data without the dependent variable, i.e.,
$D_{\mathrm{te}}'=\{(i, x'_i)\}_{i=1}^{n_\mathrm{te}}$. 

\emph{Segments} of the data are defined by tuples $s = (a, b)$ where
$a$ and $b$ are the endpoints of the segment such that $a\leq b$. We write
${D}_{|s}$ to denote the triplets in $D=\{(i,x_i,y_i)\}_{i=1}^n$ such that
the time index $i$ belongs to the segment $s$, i.e.,
$${D}_{|s}=\{ (i,x_i,y_i) \mid a\leq i\leq b\}.$$

Assume that we are given a
\emph{regression function} $f : \mathbb{R}^m \mapsto
\mathbb{R}$ trained using $D_\mathrm{tr}$. The function $f$ estimates the value of the
dependent variable at time $i$ given the covariates, i.e., $y'_i \approx
\hat y'_i = f(x'_i)$. The \emph{generalization error} of $f$ on the data set $D=\{(i,x'_i,y'_i)\}_{i=1}^n$ is
defined as
\begin{equation}
\label{eq:generalisationerror}
  \mathrm{RMSE}(f, D)  = \left(\sum\nolimits_{i=1}^{n}{\left[  f(x'_i) - y'_i\right]^2/n}\right)^{1/2},
\end{equation}
i.e., we consider the mean squared error.
In this paper we consider the following problem:
\begin{problem}
\label{prob:main}
  Given a regression function
  $f$ trained using the dataset $D_\mathrm{tr}$,
  and a threshold~$\sigma$, predict whether the generalization error
  $E$ of $f$ on the testing data~$D$ as defined by Eq.~\eqref{eq:generalisationerror}
  satisfies $E \ge \sigma$ when only the reduced testing data
  $D'$ is known and the true dependent variable $y'_i$, $i \in
  [n]$, is unknown.
\end{problem}

\subsection{Overview of the main idea}\label{sec:idea}

As discussed above in the introduction,
without the ground truth we can only detect virtual concept drift that
occurs as a consequence of changes in the covariate distribution
$p(x)$. We therefore need a distance measure $d(x)$ that measures how
``far'' a vector $x$ is from the data $D_\mathrm{tr}$ that was used to
train the regressor. Small values of $d(x)$ (which we will later call
the {\em concept drift indicator} variable) mean that we are close to
the training data and the regressor function should be reliable, while
a large value of $d(x)$ means that we have moved away from the
training data, after which the regression estimate may be inaccurate.

It is possible to list some properties that a good distance measure
should have.  On one hand, we are only interested in the changes in
the covariate distribution $p(x)$ that may affect the behavior of the
regression. For example, if there are attributes that the regressor
does not use, then changes in the distribution of that attribute alone
should not be relevant.
On the other hand, if a changed (i.e., drifted) 
attribute is important for the output of the regressor, then changes
in this attribute may cause concept drift and the value of $d(x)$
should be large.

We propose to define this distance measure as follows. We first train
different regression functions, say $f$ and $f'$, on different subsets
of the training data. We then define the distance measure to be the
difference between the predictions of these two functions, e.g.,
$d(x)=[f(x)-f'(x)]^2$. The details how we select the subsets and
compute the difference are given later in Section \ref{ssec:drifter}.

We can immediately observe that this kind of a distance measure has
the suitable property that if some attributes are independent of the
dependent variable, then they will not affect the behavior of the
regression functions and, hence, the distance measure $d$ is not
sensitive to them. In the next section we show that at least in the
case of a simple linear model, the resulting measure is, in fact,
monotonically related to the expected quadratic error of the regression
function.

\subsection{Theoretical motivation}
\label{sec:tmotivation}

In this section, we show that our method can be used to approximate
the ground truth error for an ordinary least squares (OLS) linear
regression model. Assume, that our covariates $x_i\in{\mathbb{R}}^m$
have been sampled from some distribution, for which the expected
values and the covariance matrix exists with the first term being the
intercept, or $x_{i1}=1$. Hence, we rule out, e.g., the Cauchy
distribution for which the expected value and variance are undefined.
Given the parameter vector $\beta\in{\mathbb{R}}^m$ and the variance
$\sigma_y^2$, the dependent variable is given
by $$y_i=\beta^Tx_i+\epsilon_i,$$ where $\epsilon_i$ are independent
random variables with zero mean and variance of $\sigma_y^2$.

Now, assume that we have trained an OLS linear regression model,
parametrized by $\hat\beta$, on a dataset of size $n$ and obtained a
linear model $f$, and that we have also trained a different linear
model on an independently sampled dataset of size $n'$ and obtained a
linear model $f'$ parametrized by $\hat\beta'$, respectively. For a
given $x$, the estimates of the dependent variable $y$ are then given
by $\hat y=f(x)=\hat\beta^Tx$ and $\hat y'=f'(x)=\hat\beta'^Tx$,
respectively.

We now prove the following theorem.
\begin{theorem}\label{thm:monotonic}
  Given the definitions above,
  the expected mean squared error $E\left[(f(x)-y)^2\right]$ is
  monotonically related to the expectation of the squared difference between the two regressors $f$ and $f'$, 
  i.e., $E\left[(f(x)-f'(x))^2\right]$, by the following equation to a leading order in $n^{-1}$
  and $n'^{-1}$:
  \begin{equation}\label{eq:monotonic}
    E\left[(f(x)-y)^2\right]=(1+n/n')^{-1}E\left[(f(x)-f'(x))^2\right]+\sigma_y^2.
  \end{equation}
\end{theorem}
\begin{proof}
The translation of the covariates can be absorbed in the intercept
term of the parameter $\beta$ and the rotation can be compensated by
rotating the remainder of the vector $\beta$.  We can therefore,
without loss of generality, assume that the distribution from which
the covariates have been sampled has been centered so that all terms
except the intercept have an expectation of zero, or $x_{i1}=1$ and
$E[x_{ij}]=0$ for all $j\ne 1$. We can further assume that the axes of
the covariates have been rotated so that they are uncorrelated and
satisfy $$E\left[x_{ij}x_{ik}\right]=\sigma_{xj}^2\delta_{jk},$$ where
the Kronecker delta satisfies $\delta_{jk}=1$ if $j=k$ and
$\delta_{jk}=0$ otherwise.

Now, for a dataset of size $n$, the OLS estimate of $\beta$, denoted by
$\hat\beta$, is a random variable that obeys a distribution with a
mean of $\beta$ and a covariance given by $n^{-1}\Sigma$, where
\begin{equation}\label{eq:sigma}
\Sigma=\sigma_y^2{\rm{diag}}(1,\sigma_{x2}^{-2},\ldots,\sigma_{xm}^{-2})
+{\cal{O}}(n^{-1}),
\end{equation}
where the terms of the order $n^{-1}$ or smaller have been included in
${\cal{O}}(n^{-1})$. The covariance $n^{-1}\Sigma$ is therefore
proportional to $n^{-1}$ and hence, at the limit of a large dataset we
obtain the correct linear model, i.e.,
$\lim_{n\rightarrow\infty}{\hat\beta}=\beta$. For finite data there is
always an error in the estimate of $\hat\beta$. The expected
estimation error is larger for small data, i.e., if $n$ is small.

It follows from Eq.~\eqref{eq:sigma} that the
expected mean squared error for a model evaluated at $x$ is given by
\begin{equation}\label{eq:mse1}
  E\left[(f(x)-y)^2\right]=
  x^T\left(n^{-1}\Sigma\right)x+\sigma_y^2,
\end{equation}
and the expected quadratic difference between the linear model estimates is given by
\begin{equation}\label{eq:mse2}
  E\left[(f(x)-f'(x))^2\right]=
  x^T\left[\left(n^{-1}+n'^{-1}\right)\Sigma\right]x.
\end{equation}
We can solve for $x^T\Sigma x$ from Eq.~\eqref{eq:mse2}
and insert it in Eq.~\eqref{eq:mse1}, from which Eq.~\eqref{eq:monotonic}
follows.
\end{proof}

We hence postulate that the squared differences between the estimates
given by regressors trained on different subsets of the data --- either 
sampled randomly or obtained by other means --- can be
used to estimate the mean squared error even when the ground truth
(the value of $y$) is not known. Of course, in most interesting cases
the regression functions are not linear, but as we show later in
Sec.~\ref{sec:experiments}, the idea works also for real datasets and
complex non-linear regression models.

Our claim is therefore that the difference between the estimates of
regressors trained on different subsets of the data in the point $x$
defines a distance function which can be evaluated even when the
ground truth is unknown. If a data point $x$ is close to the data
points used to train the regressors the distance should be small. On
the other hand, if the data point is far away from the data used to
train the regressors, the predictions of the regressors diverge and
the distance and also the prediction error will be larger.

\begin{figure}[t]
  \centering
\includegraphics[width = 0.75\textwidth, trim=2mm 4mm 10mm 25mm, clip]{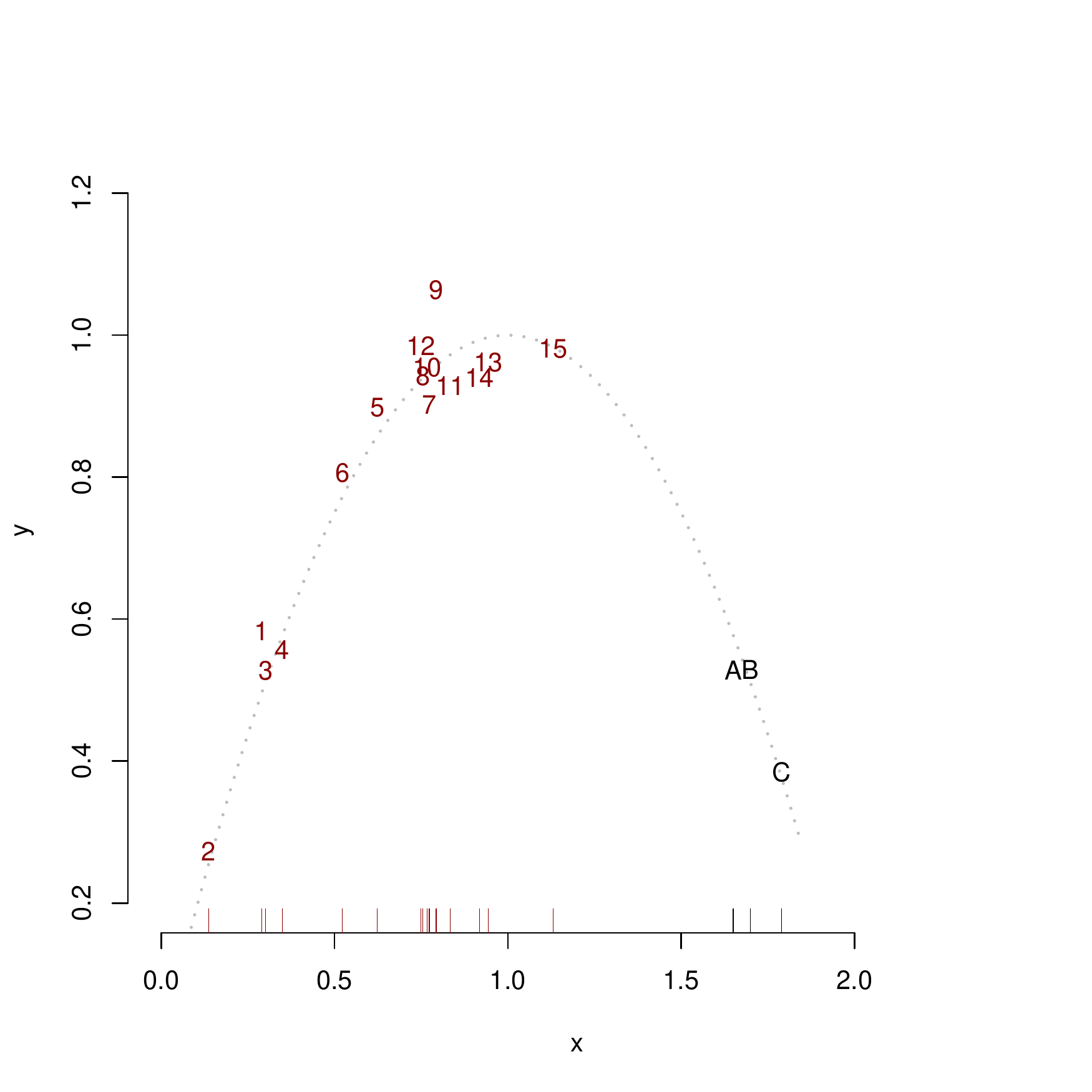}
\caption{Example data set with covariate $x$ and response variable $y$. The training data $D_{15}$ shown with numbers and the testing data $D_{ABC}$ shown with letters. \label{fig:drifter-toydata}}
\end{figure}

\section{The Drifter Algorithm} \label{ssec:drifter}

In this section we describe our algorithm
for detecting concept drift when the ground truth is
unknown. We start with a simple data set shown in
Fig.~\ref{fig:drifter-toydata} and go through the general idea using
this data as an example. We then continue by providing the
algorithmic details of the training and testing phases of our
algorithm (Sections~\ref{ssec:training} and \ref{ssec:testing},
respectively), and a discussion of how to select a suitable value for
the drift detection threshold in Sec.~\ref{ssec:threshold}.

Let us consider $D_{15}=\{(i, x_i, y_i)\}_{i=1}^{15}$ consisting of 15
data points, with the one-dimensional covariate $x_i$ and response
$y_i$, as shown in Fig.~\ref{fig:drifter-toydata}, and assume that the
data set $D_{15}$ has been used to train a Support Vector Machine
(SVM) regressor $f$. The SVM model estimate of $y$ is shown with black
solid line in Fig.~\ref{fig:drifter-overview2d}.  Our testing data
$D_{ABC}$ then consists of the data points labeled with $A,B$, and $C$
in Fig.~\ref{fig:drifter-toydata}, and we want to estimate the
generalization error of $f$, when we only have access to the
covariates of $D_{ABC}$.

\begin{figure}[t]
  \centering
\includegraphics[width = 0.75\textwidth, trim=2mm 4mm 10mm 20mm, clip]{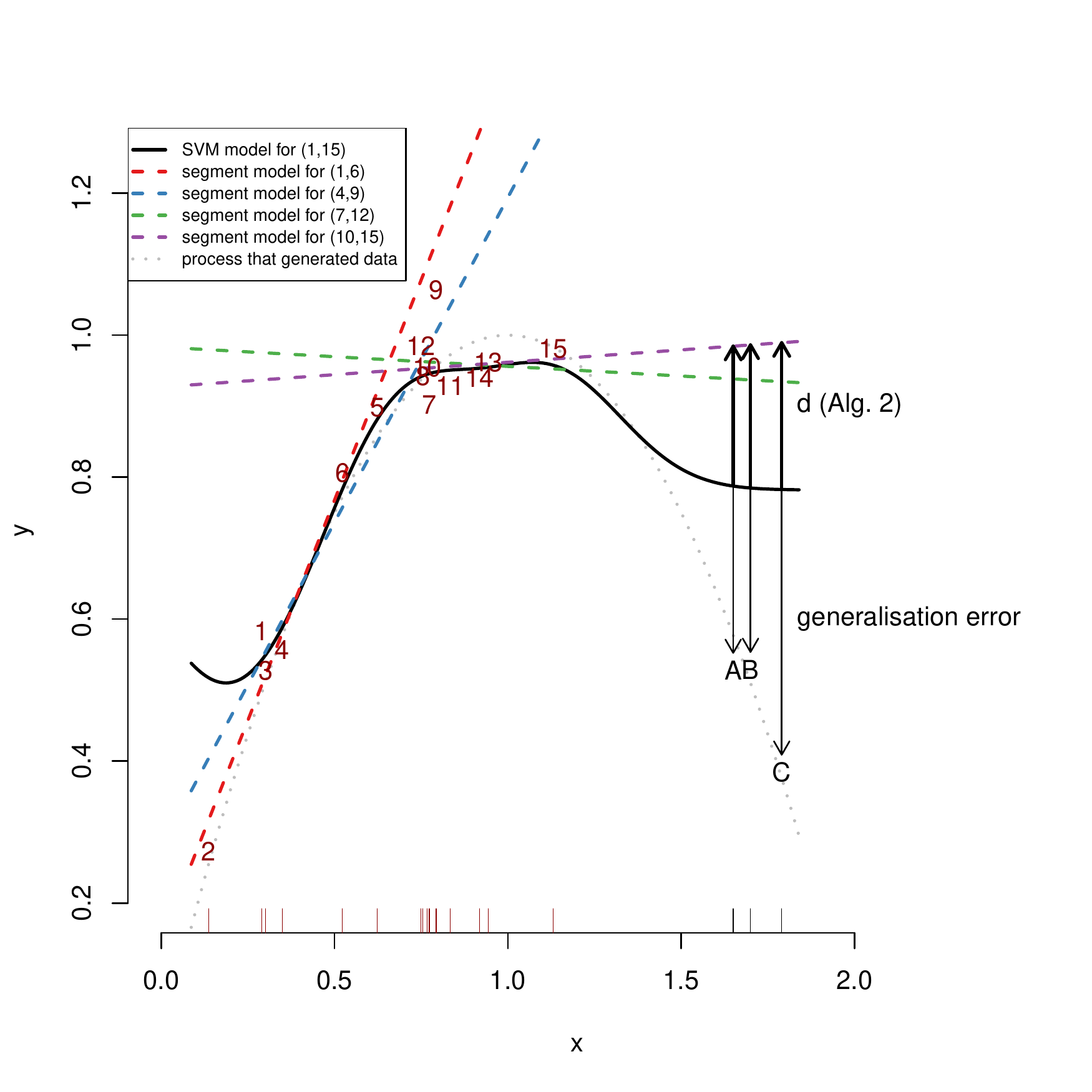}
\caption{The models trained using $D_{15}$ and subsequences of it. \label{fig:drifter-overview2d}}
\end{figure}

\begin{figure}[t]
  \centering
\includegraphics[width = 0.8\textwidth, trim=10mm 22mm 15mm 15mm, clip]{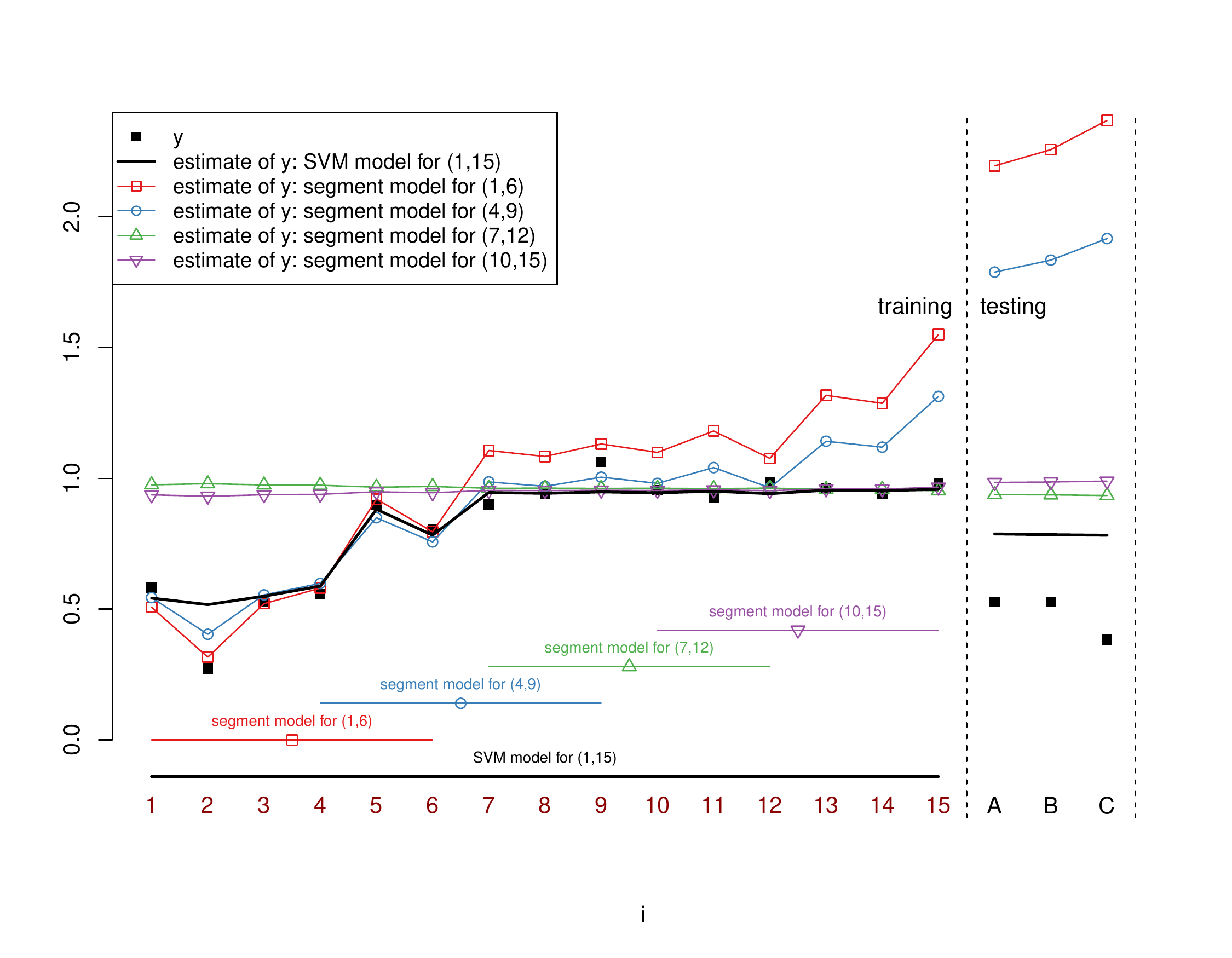}
\caption{The response variable $y$ and the estimates of $y$ using different models for the training data $D_{15}$ and the testing data $D_{ABC}$.\label{fig:drifter-overview1d}}
\end{figure}

Now, recalling Thm.~\ref{thm:monotonic}, we can estimate the
generalization error using the difference of two regressors. Thus,
instead of considering $\mathrm{RMSE}(f, D_{ABC})$, which we cannot
compute without knowing the response variable, we estimate the error
using the terms $[f(x)-f'(x)]^2$ instead of $[f(x)-y]^2$ for each
$(i,x,y)\in D_{ABC}$, where $f'$ is another regressor function. 

To obtain a suitable $f'$, we train several regression functions,
called {\em segment models} using subsequences of the data, i.e.,
segments. Our intuition here is that due to autocorrelation, a
subsequence is more likely to contain samples from the same
distribution of covariates. We call a distribution of covariates in a
subsequence a {\em concept}. In our example, we consider segments
$s_1=(1,6)$, $s_2=(4,9)$, $s_3=(7,12)$, and $s_4=(10,15)$, which are
overlapping, and we train the segment models using OLS regression.
With overlapping segments, we aim towards robustness, i.e., we assume
it to be unlikely that the overlapping
segmentation splits very clear concepts in a way that they would not
be present in any of the segments.
The linear segment models are shown in
Fig.~\ref{fig:drifter-overview2d} using colored dashed lines, and the
estimates are shown in Fig.~\ref{fig:drifter-overview1d} using the same
colors. We observe that the segment models are good estimates for the SVM model
on their respective training segments.

Using the segment models $f_1$,\ldots, $f_4$, we can compute an
estimate of the generalization error using the terms $[f(x)-f_i(x)]^2$
instead of $[f(x)-y]^2$ for each $f_i$. This allows us then to compute some
statistics based on the estimates obtained from this ensemble of
segment models. 
Here, we choose the statistics, namely the {\em concept drift indicator} value, to be the {\em
second smallest error}. The intuition is that if the test data
resembled some concept in the training data, and an overlapping
segmentation scheme was used, at least two of the segment models should
provide a reasonably small indicator value. Hence, if there exists
only a single small indicator value, it could well be due to chance.
Thus, using the second smallest value as the indicator value increases
the robustness of the method.

In Fig.~\ref{fig:drifter-overview1d} we visualize the terms
$[f(x)-f_4(x)]^2$ for each data point in $D_{ABC}$, where $f_4$
trained using the segment $(10,15)$ is the second-best linear model
for $D_{ABC}$.
Since our estimate for the generalization error of $f$ in $D_{ABC}$ is
large even when using the second-best linear model $f_4$, we conclude
that indeed it is likely that there is concept drift in $D_{ABC}$.
 
In the following, we formalize the ideas presented in the discussion
above, and provide a detailed description of the \emph{training} and
\emph{testing} phases of the \drifter algorithm.

\begin{algorithm2e}[t]
  \DontPrintSemicolon
  \SetKwProg{Fn}{Function}{}{}
  \SetKwInOut{params}{Input}\SetKwInOut{output}{Output}
  \SetKwFunction{trainreg}{TrainRegressor}
  \SetKwFunction{trainseg}{train\_f}
  \SetKwFunction{driftertrain}{drifter\_train}
   \SetKwFunction{testfunc}{MakeTesterFunc}
  \params{training data $D_\mathrm{tr}=\{(i,x_i,y_i)\}_{i=1}^{n_\mathrm{tr}}$,  segmentation $S$,\\
  function for training segment models \trainseg}
    \output{array $F$ of regression functions}
\Fn{
    \driftertrain{$D_\mathrm{tr}$,  $S$, \trainseg}}{
    $(s_1,\ldots, s_k)\leftarrow S$\;
   \For{$i \in [k]$}{  \label{alg:tr:2}
     $f_i \leftarrow$  \trainseg{${D_\mathrm{tr}}_{|s_i}$} \label{alg:tr:3} \;
    }
  \Return{$F\leftarrow(f_1,\ldots, f_k)$}
  }
  \caption{Training phase of {\small \texttt{drifter}}.\label{alg:drifter:train}}
\end{algorithm2e}

\subsection{Training phase\label{ssec:training}}

In the training phase of  \drifter (Alg.~\ref{alg:drifter:train}), we train the segment models
for subsequences, i.e., segments, of the training data. 
As input, we assume the training data $D_\mathrm{tr}=\{(i,x_i,y_i)\}_{i=1}^{n_\mathrm{tr}}$, 
a segmentation $S$ of $[n_\mathrm{tr}]$, and 
a function {\small \texttt{train\_f}}  for training segment models.

Hence, we assume that the user provides a segmentation $S$ of $[n_\mathrm{tr}]$ 
such that when the segment models
are trained, the data used to train a model approximately corresponds to only one \emph{concept}, i.e., the
models ``specialize'' in different concepts. Here there might, of course,
be overlap so that multiple models are trained using the same
concept. 
We show in Sec.~\ref{sec:experiments} that using a scheme in which   
the segmentation consists of equally-sized segments of length $l_\mathrm{tr}$ with 50\% overlap,
the \drifter method is quite robust with
respect to the choice of $l_\mathrm{tr}$, i.e., just selecting a reasonably small
segment length $l_\mathrm{tr}$ generally makes the method perform well and provides a simple baseline approach for selecting a segmentation.
However, the segments could well be of varying length or non-overlapping. For instance,
by using a segmentation that is a solution to the basis segmentation problem \cite{bingham2006segmentation},
one would know that each segment can be approximated with linear combinations
of the basis vectors.
  
The training phase essentially consists of training a regression function
$f_i$ for each segment $s_i \in S$  using ${D_\mathrm{tr}}_{|s_i}$ (lines \ref{alg:tr:2}--\ref{alg:tr:3} in Alg.~\ref{alg:drifter:train}).
These regression functions are the
\emph{segment models}. Note, that the
model family of the segment models is chosen by the user and provided as input to Alg.~\ref{alg:drifter:train}.
Natural choices are, e.g.,
linear regression models
or, if known, the same model family using which the function $f$, that will be used in the testing phase, was trained.

\begin{algorithm2e}[t]
  \DontPrintSemicolon
  \SetKwProg{Fn}{Function}{}{}
  \SetKwInOut{params}{Input}\SetKwInOut{output}{Output}
  \SetKwFunction{trainreg}{TrainRegressor}
  \SetKwFunction{driftereval}{drifter\_test}
  \SetKwFunction{sort}{sort\_increasing}
  \SetKwFunction{threshold}{compute\_threshold}
  \params{testing data
  $D'_\mathrm{te}=\{(j,x'_j)\}_{j=1}^{n_\mathrm{te}} $,   model $f: \mathbb{R}^m \mapsto \mathbb{R}$,\\
  array of segment models $F$, integer $n_\mathrm{ind}$
  }
   \output{concept drift indicator variable $d$}
  \Fn{
    \driftereval{$D'_\mathrm{te}$,  $f$, $F$, $n_\mathrm{ind}$}} { 
    $(f_1,\ldots, f_k)\leftarrow F$\;
   \For{$i \in [k]$}{   \label{alg:te:3}
     $z_i\leftarrow   \mathrm{RMSE}^*(f, f_i, D'_\mathrm{te})$ \label{alg:te:4}\;
    }
   \Return{$d \leftarrow $ \sort{$z_1,\ldots, z_k$}$[n_\mathrm{ind}]$  \label{alg:te:5}}
}
   \caption{Testing phase  of {\small \texttt{drifter}}.}
  \label{alg:drifter:test}
\end{algorithm2e}

\subsection{Testing phase\label{ssec:testing}}

The tester function of \drifter (Alg.~\ref{alg:drifter:test}) takes as input the testing data $D'_\mathrm{te}$, the
model~$f$, the segment models $F$ from Alg.~\ref{alg:drifter:train}, and an integer
$n_\mathrm{ind}$ (the {\em indicator index order}). 
For each of the $k$ segment models $f_i$, we then determine the RMSE
between the predictions from $f_i$ and $f$ on the test data (lines
\ref{alg:te:3}--\ref{alg:te:4}),  i.e., we compute 
\begin{equation}
\label{eq:errorindex}
  \mathrm{RMSE}^*(f, f_i, D'_\mathrm{te})  =
  \left(\sum\nolimits_{j=1}^{n_\mathrm{te}}{\left[  f(x'_j) -f_i(x'_j)\right]^2}/n_\mathrm{te}\right)^{1/2},
\end{equation}
where $D'_\mathrm{te}=\{(j,x'_j)\}_{j=1}^{n_\mathrm{te}} $.
This gives us $k$ values $z_i$ (line \ref{alg:te:4}) estimating the generalization error, 
and we  then choose
the  $n_\mathrm{ind}$th smallest value as the value for the {\em concept drift indicator} variable $d$
(line \ref{alg:te:5}).
If this value is large, then 
the predictions from the full model on the test data in question can be
unreliable.

In this paper, we use $n_\mathrm{ind}=2$  by default.
The intuition behind this choice is that,
due to the overlapping segmentation scheme we use, it is reasonable to assume that
at least two of the segment models should have small values for $z_i$'s,
if the testing data has no concept drift,
while a single small value for $z_i$ could still occur by chance.

In the testing phase, there is an implicit assumption regarding the length of the testing data $n_\mathrm{te}$, i.e., it should hold that $n_\mathrm{te}\leq l_\mathrm{tr}$, where $l_\mathrm{tr}$ is the length of a segment in the training phase. This is because we assume that the segment models are trained to model concepts present in the training data. Hence, if $n_\mathrm{te}\gg l_\mathrm{tr}$, the testing data might consist of several concepts, resulting in a large value for the concept drift indicator value $d$, 
implying concept drift even in the absence of such.
This can be easily prevented, e.g., as we do in the experimental evaluation in Sec.~\ref{sec:experiments}, by dividing 
the testing data into smaller (non-overlapping) \emph{test segments} of length $l_\mathrm{te}\leq l_\mathrm{tr}$ 
and calling the 
tester function
(Alg.~\ref{alg:drifter:test}) for each of the test segments.  
Thus, in this way we obtain a concept drift indicator value for each smaller segment in the testing data. 


\subsection{Selection of the drift detection threshold\label{ssec:threshold}}

In order to solve Prob.~\ref{prob:main}, we still need a suitable concept drift detection threshold $\delta$, 
i.e., we need a way to define a threshold for the concept drift indicator variable $d$ (Alg.~\ref{alg:drifter:test}) 
that estimates the threshold $\sigma$ for the generalization error in Prob.~\ref{prob:main}.

As a general observation we note that a good concept drift detection threshold $\delta$ 
depends both on the dataset and the application at hand. In this section, we propose a general method
for obtaining a threshold, which according to our empirical evaluation (see Sec.~\ref{sec:experiments}) 
performs well in practice for the datasets used in this paper.
However, we note that a user knowledgeable of the particular
data and the application can use this knowledge to select and 
potentially adjust a better threshold for the data and the application at hand.

One could also, e.g., make use of a validation set $D_\mathrm{val}$ with known ground truth values
(not used in training of $f$) and
compute the generalization error $\mathrm{RMSE}(f,D_\mathrm{val})$.
Then a suitable  threshold~$\delta$ could be determined, e.g., using \emph{receiver
operating characteristics} (ROC) analysis, which
makes it possible to balance the tradeoff between false positives and
false negatives \cite{fawcett2006}.
Note, however, that one needs to assume that there is no concept drift in the validation set  $D_\mathrm{val}$, and
also one should consider, e.g., crossvalidation when training $f$ to prevent overfitting which could heavily affect
  $\mathrm{RMSE}(f,D_\mathrm{val})$.

We now describe how to compute the threshold $\delta$ using only the training data. 
We first split the training dataset into $\lfloor n_\mathrm{tr}/n_\mathrm{te}\rfloor$
(non-overlapping) segments of the same length as the testing data. 
Then, we compute the concept drift indicator value $d_i$ for each of these
segments $s_i$ in the training data ${D_\mathrm{tr}}_{|s_i}$
using Alg.~\ref{alg:drifter:test}, i.e.,  
$d_i=\mathtt{drifter\_test}({D_\mathrm{tr}}_{|s_i}, f, F, n_\mathrm{ind})$.
 We then choose an concept drift detection threshold $\delta$
 by using the mean and standard deviation of the indicator values of these segments
 \begin{equation}
 \label{eq:delta}
 \delta= \mathrm{mean}(d_i)+c\times \mathrm{sd}(d_i), 
 \end{equation}  
 where $c$ is a {\em constant} multiplier of choice.
The optimal value of $c$ depends on the properties of a particular
dataset, but our empirical evaluation of the effect of varying $c$
(see Fig.~\ref{fig:best_cs} in Sec.~\ref{sec:experiments}), shows that
the performance with respect to the $\mathit{F1}$-score (see
Eq.~\eqref{eq:f1} for definition) is not overly sensitive with respect
to the choice of $c$. For example, $c=5$ works well for all the
datasets we used.

\subsection{Using  \drifter  to solve Prob.~\ref{prob:main}\label{ssec:fulldrifter}}

We now summarize, how the \drifter method is used in practice to
solve Prob.~\ref{prob:main}. Assume that a model $f$ has been trained
using $D_\mathrm{tr}=\{(i,x_i,y_i)\}_{i=1}^{n_\mathrm{tr}}$, and we
know that the {\em concept length} in the training data is
$l_\mathrm{tr}$. We use this knowledge to form a segmentation $S$ of
$[n_\mathrm{tr}]$ such that there are $k$ segments of length
$l_\mathrm{tr}$. We also need to make a choice for the model family
using which we should train the segment models (function
\texttt{train\_f}). In practice, linear regression models seem to consistently
perform well (see, Sec.~\ref{sec:experiments}).  Then,
the {\em training phase} consists of a call to
Alg.~\ref{alg:drifter:train} to obtain an ensemble $F=(f_1, \ldots,
f_k)$ of segment models.

Once the segment models have been trained, we can readily use these to
detect concept drift in the testing data
$D'_\mathrm{te}=\{(i,x_i)\}_{i=1}^{n_\mathrm{te}}$. In the {\em
  testing phase}, we should call Alg.~\ref{alg:drifter:test} for a
testing data with $n_\mathrm{te}\leq l_\mathrm{tr}$, where
$l_\mathrm{tr}$ is the segments length used to train the segment
models~$F$ in Alg.~\ref{alg:drifter:train}. In practice, this is
achieved by splitting the testing data into small segments of length
$l_\mathrm{te}$ (e.g., we use a constant $l_\mathrm{te}=15$ in the
experimental evaluation in Sec.\ref{sec:experiments}) and calling
Alg.~\ref{alg:drifter:test} for each small test segment individually.

If we have split the testing data into $k'$ segments, and obtained a
vector of concept drift indicator values $(d_1, \ldots, d_{k'})$ using
Alg.~\ref{alg:drifter:test}, we can then compare these values to the
concept drift detection threshold $\delta$, which is either
user-specified or obtained using the approach described in
Sec.~\ref{ssec:threshold}, and classify each segment in the testing
data, either as a segment exhibiting concept drift ($d_i\geq \delta$)
or not ($d_i<\delta$).

We observe that the time complexity of \drifter
is dominated by the training phase, where we need to train $k$
regressors using data of size $\bigohpar{n_\mathrm{tr}/k}$ and
dimensionality $m$. For OLS regression, e.g., the complexity of
training one segment model is hence $\bigohpar{n_\mathrm{tr}m^2/k}$
and hence the complexity of the training phase is $\bigohpar{nm^2}$.


\section{Experiments} \label{sec:experiments}

In this section we experimentally evaluate {\small \texttt{drifter}} in detection of concept drift.
We first present the datasets and the regressors used (Sec. \ref{ssec:datasets}) and discuss generalization error and the default parameters used in the experiments (Sec. \ref{ssec:parameters}). In Secs.~\ref{ssec:sel-concept-model-family} and  \ref{ssec:sel-c}
we pin down suitable combinations of the remaining parameters of {\small \texttt{drifter}}.
In Sec. \ref{ssec:scalability} we assess the runtime scalability of {\small \texttt{drifter}}
 on synthetic data, and finally in
Sec. \ref{ssec:detection} we look at how {\small \texttt{drifter}} finds concept drift on our considered dataset and regression function
combinations.

The experiments were run using R (version 3.5.3) \cite{Rproject} on
a high-performance cluster  \cite{fcgi}
(2 cores from an Intel Xeon E5-2680 2.4 GHz with 256 Gb RAM).
An implementation of the \drifter algorithm and
the code for the experiments presented in this paper 
has been released as open-source software \cite{drifter}.

\subsection{Datasets and regressors}\label{ssec:datasets}
We use the datasets described in Tab.~\ref{tab:datasets} in our
experiments. A brief description of each dataset and the regressor
trained using it is provided below. During preprocessing we removed
rows with missing values, and transformed factors into numerical
values. For each dataset, we then use the first 50\% of the data as
the training set, and the remaining 50\% as a testing dataset, i.e.,
$n_\mathrm{tr}=\lfloor 0.5 n\rfloor$ and $n_\mathrm{te}=\lceil
0.5n\rceil$. We split the testing data into non-overlapping {\em test
  segments} of length $l_\mathrm{te}=15$ as described in
Sec.~\ref{ssec:fulldrifter}.

\begin{table}[t]
 \centering
  \caption{Datasets used in the experiments.\label{tab:datasets}}
 \begin{ruledtabular}
 \begin{tabular}{lrrcc}
 \textbf{Name} & $n$ & $m$ & \textbf{Target} & \textbf{Regressor} \\
\midrule
  \airquality & 7355 & 11 & CO(GT) & SVM \\
\airline & 38042 & 8 & Arrival delay & RF\\
\bikesharing  & 731 & 8 & Rental count & LM\\
\syntheticns($n$,$m$) & $n$ & $m$& $y \in \mathbb{R}$ &  LM, SVM, RF
\end{tabular}
\end{ruledtabular}
\end{table}

\paragraph{Air quality data} 
The \airquality dataset \cite{devito:2008} contains hourly air quality sensor
  measurements spanning approximately one year. We preprocessed the data by removing
  rows with missing data as well as the attribute NMHC(GT)
  containing mostly
  missing data.
  We use the first half of the data as the training set
  and train a regressor $f_\mathrm{AQ}$ for
   hourly averaged concentrations of carbon monoxide CO(GT) using Support Vector Machine (SVM) from
  the \emph{`e1071'} R package with default parameters.

\paragraph{Flight delay data}
The \airline dataset \cite{airline:kaggle} contains data related to flight delays
collected and published by the U.S. Department of Transportation's Bureau of Transportation Statistics.
We used the arrival delay variable as the target variable and selected
  a subset of the other attributes as covariates (namely, departure delay, day of the week,
  origin airport, airline, departure time, destination airport, distance, and scheduled arrival). In order
  to keep computation time manageable we only used every 150th
  sample.
  We used the first half of the data as the training set
  and trained a regressor $f_\mathrm{AL}$ for
  the arrival delay using Random Forest (RF) from the \emph{`randomForest'} R package 
   with default parameters.

\paragraph{Bike rental data}
The \bikesharing dataset \cite{bikesharing} contains daily counts of bike rentals
  and associated covariates related to weather and date types for a period of about two years.
  As covariates we used the attributes for
  holiday, weekday, working day, weather situation, temperature related variables, humidity, and windspeed.
  Hence, inherently drifting covariates such as date
  and season were removed.
  Exploratory analysis indicated \emph{real concept drift} to be present in
  the form of an increasing trend in the counts of bike rentals.
  Thus, we prepared an alternative version of the data by
  removing this trend in the testing data
  by multiplying each $y_i'\in D_\mathrm{te}$ by 
  $$m=\mathrm{mean}_{y\in D_\mathrm{tr}}(y)/\mathrm{mean}_{y'\in D_\mathrm{te}}(y').$$
  Hence, in the dataset \bikesharingr we use the original rental counts $y_i$ (with real concept drift present), whereas in the dataset \bikesharingd
  we use the modified rental counts $m\times y_i$ (with real concept drift removed).
 We then used the first half of both datasets as the training set
  and trained  OLS linear regression models (LM)  $f_\mathrm{B(raw)}$ and  $f_\mathrm{B(detr)}$ for predicting the rental counts for
  \bikesharingr and  \bikesharingd, respectively.

\paragraph{Synthetic data}
We did not find adequate existing methods for generating synthetic regression datasets
containing only virtual concept drift, and thus developed a new method for
constructing the data. The \syntheticns($n$,$m$) data we used is constructed as follows. The covariate matrix
$X \in \mathbb{R}^{n \times m}$ is sampled columnwise from $\text{AR}(1)$ with correlation length
$h$, defined as the number of steps after which the expected autocorrelation drops to $0.5$, and the amplitude $amp$.
The elements of a noise vector $e \in \mathbb{R}^{n}$ are sampled from a normal distribution $N(0,\sigma_N^2)$.
The target variable is then constructed
as $Y = g(X^T)+e$, where we use $g=\sin$ to introduce non-linearity.
Here, we use $h=150$, 
$amp = 1$, and i.i.d. noise $e$ sampled from $N(0,\sigma_N^2)$, where
$\sigma_N=0.3$. 

In the scalability experiments (reported in Sec.~\ref{ssec:scalability}) we vary  the data dimensions $n$ and $m$ when generating datasets \syntheticns($n$,$m$).
In the remaining experiments,  we use the dataset \syntheticns($2000$,$5$),  i.e., a 5-dimensional dataset $X$ of length $n=2000$ samples, 
where we  add 
a concept drift component at $[1700,1800]$ by
using $amp = 5$
during this period.
We trained LM, RF, and SVM regressors with the \synthetic data.

\subsection{Generalization error threshold and parameters of \texttt{drifter}} \label{ssec:parameters}
The datasets we use do not have predefined ground truth values
and we hence first need to define what constitutes concept drift
in the test datasets.
The user should choose the threshold $\sigma$:
in some applications a larger generalization
error could be tolerated, while
in some other applications the user might want to be alerted already about smaller errors.
In the absence of a user, we determined the error threshold $\sigma_\mathrm{emp}$ for the
datasets as follows.
We used 5-fold cross-validation,
where we randomly split the training data into five folds, and estimated the value
of the $i$th dependent variable $y_i$ by a regressor trained on the four folds that do not contain $i$,
thereby obtaining a vector of estimates $\hat y_i$ for all $i\in[n_{\mathrm{tr}}]$.
We then computed the generalization error for the training data
as in Eq. \eqref{eq:generalisationerror} and then chose
$$\sigma_\mathrm{emp} =2\times\left(\sum\nolimits_{i=1}^{n_{\mathrm{tr}}}{\left(
\hat y_i-y_i\right)^2/n_{\mathrm{tr}}}\right)^{1/2}.$$
Then, all values exceeding $\sigma_\mathrm{emp}$ in
the test dataset are considered concept drift. While this cross-validation
procedure does not fully account for possible autocorrelation in the training data
we found that in our datasets it gives a reasonable estimate of the generalization error
in the absence of concept drift.

To assess what a suitable value for $c$ would be in our proposed scheme for selecting
the detection threshold $\delta$ (Sec. \ref{ssec:threshold})  
and to assess how well the scheme works in practice,
we compute the ``optimal'' detection threshold $\delta_\mathrm{opt}$
in terms of the $\mathit{F1}$-score for a given error threshold $\sigma_\mathrm{emp}$ as follows.
We vary the concept drift detection threshold $\delta$ 
and
evaluate the true and false positives rates on the test dataset, allowing us to form a ROC curve.
We then pick the $\delta_\mathrm{opt}$ maximizing the
$\mathit{F1}$-score:
\begin{equation}
\label{eq:f1}
\mathit{F1} = 2 \mathit{TP} / (2\mathit{TP} + \mathit{FP} + \mathit{FN}),
\end{equation}
where $\mathit{TP}$ are true positives, $\mathit{FP}$ are false positives and $\mathit{FN}$ are
false negatives.

For the other parameters, we 
use in the training phase the segmentation scheme with 50\% overlap between consecutive segments. 
In the testing phase, we split the testing data into non-overlapping segments of fixed length ($l_\mathrm{te}=15$), and evaluate the concept drift indicator value on each test segment using
$n_\mathrm{ind}=2$.
In preliminary experiments, we also tested a segmentation scheme with no overlap between segments in the training phase,
and values $n_\mathrm{ind}\in\{1,2,3,5\}$.
The effect of these parameter options was rather small in practice, and
we chose the values using which the \drifter method performed most robustly in detecting virtual concept drift for our datasets.

\subsection{Effect of concept length and segment models}\label{ssec:sel-concept-model-family}

We next investigate the effects of the remaining input parameters,
i.e., (i) the constant $c$ in Eq.~\eqref{eq:delta}, (ii) the concept
length (i.e., the segment length $l_\mathrm{tr}$ in the training
phase), and (iii) the effect of the model family for the segment
models.  

We varied $k=\lfloor n_\mathrm{tr}/l_\mathrm{tr}\rfloor$, which means that
there are $2k-1$ segments in the training phase in the overlapping segmentation scheme.
The maximum value for $k$ was determined by the requirement $l_\mathrm{tr}\geq l_\mathrm{te}=15$.
For each $k$ and each dataset, we determined the value for $c_\mathrm{opt}$
that leads to $\delta_\mathrm{opt}$ in Eq.~\eqref{eq:delta} maximizing the $\mathit{F1}$-score.

For the choice of the model family, we considered two cases: either the segment models were trained using the same model family as the model $f$ given as input, or linear regression was used for the segment models.
Our evaluation showed that the linear segment models consistently performed the best (both in terms of  performance, e.g., $\mathit{F1}$-score,  robustness, and  computational cost, i.e., time needed to train the models). We hence focus on utilizing linear regression models as segment models in the rest of this paper. 

We would like to point out that there is an intuitive reason why LM outperforms SVM and RF as segment models. While SVM and RF 
give accurate predictions on the training data covariate distribution, they predict constant values outside of it. 
The linear OLS regressor on the other hand gives (non-constant) linearly increasing/decreasing predictions generalization
the farther from the training data covariate distribution the testing data is.
It should also be noted that for SVM the kernel choice makes a difference in terms of generalization behavior. 
We here used a radial basis function kernel, but if a polynomial kernel or a linear kernel were used, the model would behave more like LM.
 
The results are presented in Tab.~\ref{tab:k_results}.
The table shows the \emph{number of segments} of length $l_\mathrm{te}=15$ in the testing data identified as true ($\mathit{TP}$) and false ($\mathit{FP}$) positives, and true ($\mathit{TN}$) and false negatives ($\mathit{FN}$), respectively.
We observe, that concept drift is detected with a reasonable accuracy for the \airquality, \airline and \syntheticns(2000,5) datasets in terms of $\mathit{F1}$-score, i.e., the number of true positives and negatives is high, while the number of false positives and negatives remains low.
For each of these datasets we have identified the best performing combination of $k$ and $c$
(shown with bold in Tab.~\ref{tab:k_results}), and we subsequently use these particular combinations in Sec.~\ref{ssec:detection}.

For \bikesharingr we observe negative values when optimizing $c$. 
This is due to the \emph{real concept drift in the data}, i.e.,
the bike rental counts are higher during the second year, likely due to increasing popularity of the service. 
This is an effect not present in the training data (see Sec.~\ref{ssec:detection} and Fig.~\ref{fig:res:plots2}c for details).
Since real concept drift does not affect the concept drift indicator values, it means that the optimal threshold maximizing the $\mathit{F1}$-score
would be set to a very low value. We also observe a high number of false negatives.

However, when we consider the detrended \bikesharingd dataset in which the real concept drift has been removed, we no longer observe any (virtual) concept drift in the data (and hence we cannot compute the $\mathit{F1}$-scores). We can observe that our 
'algorithm correctly handles this, i.e., all the segments in the testing data are classified as true negatives. 
For the values of \bikesharingr and \bikesharingd in Tab.~\ref{tab:k_results} we have used $\delta$ larger than the maximal value of $d$ (similarly as in Sec.~\ref{ssec:detection} and Fig.~\ref{fig:res:plots2}c,d).

\begin{table}[H]
  \centering
  \caption{The effect of segment length $l_\mathrm{tr}$, using $k=\lfloor n_\mathrm{tr}/l_\mathrm{tr}\rfloor$, on drift detection accuracy in terms of the $\mathit{F1}$-score. $c_\mathrm{opt}$ is the value using which Eq.~\eqref{eq:delta} results in $\delta_\mathrm{opt}$ maximizing the $\mathit{F1}$-score,  $\mathit{TP}$ (resp. $\mathit{FP}$) is the count of true positives (resp. false positives), and $\mathit{TN}$ (resp. $\mathit{FN}$) is the  count of true negatives (resp. false negatives).}
  \label{tab:k_results}
\begin{ruledtabular}
\begin{tabular}{lccccccccc}
 \textbf{Data} & \textbf{Full model} & \textbf{Segment model} & \textbf{k} &
$\mathbf{c_\mathrm{opt}}$ & \textbf{F1} & \textbf{TP} & \textbf{FP} & \textbf{TN} & \textbf{FN}\\
 \midrule
  \airquality & SVM & LM& 2 & 6.770& 0.735&61 & 20 & 139& 24\\
                   & &  & {\bf{10}} & {\bf{7.144}} & {\bf{0.741}} &63 & 22 & 137& 22\\
                   & &  & 20 & 7.080& 0.737&56 & 11 & 148& 29\\
                  & &  & 100 & 5.835& 0.688&54 & 18 & 141& 31\\

 \midrule
  \airline & RF & LM& {\bf{2}} &        {\bf{5.632}} & {\bf{0.786}} &11 & 1 & 1250& 5\\
                   & &  & 10   & 5.695& 0.786&11 & 1 & 1250& 5\\
                   & &  & 20   & 5.769& 0.786&11 & 1 & 1250& 5\\
                  & &  & 100   & 5.424& 0.769&10 & 0 & 1251& 6\\

 \midrule
                  \bikesharingr & LM & LM  & 2 &  --& -- & 0 & 0& 5 & 18\\
                   &      &        & 4   &--& -- & 0 & 0& 5 & 18\\
                                  & &      & 6   &--& --&0 & 0 & 5 & 18\\

 \midrule
  \bikesharingd & LM & LM  & 2 &  --& --& 0 & 0&23&  0\\
 &          &    & 4   &--& --& 0 & 0& 23&  0\\
                  & &      & 6   &--& --& 0 & 0& 23&  0\\

 \midrule
  \syntheticns(2000,5) &LM&LM&2 & 5.571& 0.737&7 & 1 & 54& 4\\
                   & &  & 10& 4.722& 0.778&7 & 0 & 55& 4\\
                  & &  & {\bf{60}} & {\bf{1.747}}& {\bf{0.857}}&9 & 1 & 54& 2\\

 \midrule
  \syntheticns(2000,5) &SVM&LM&2 &6.930& 0.769&5 & 2 & 58& 1\\
                   & &  & 10& 8.817& 0.769&5 & 2 & 58& 1\\
                  & &  & {\bf{60}} & {\bf{17.015}}& {\bf{0.833}}&5& 1 & 59& 1\\

 \midrule
  \syntheticns(2000,5) &RF&LM&2 & 5.819& 0.750&6 & 1 & 56& 3\\
                   & &  & 10& 9.649& 0.778&7 & 2 & 55& 2\\
                  & &  & {\bf{60}} & {\bf{3.883}}& {\bf{0.842}}&8 & 2 & 55& 1\\
\end{tabular}
\end{ruledtabular}
\end{table}

\begin{figure}[t]
  \centering
   \subfloat[\syntheticns(2000,5) with LM]{\includegraphics[width = 0.33\textwidth, trim=1mm 5mm 3mm 20mm, clip]{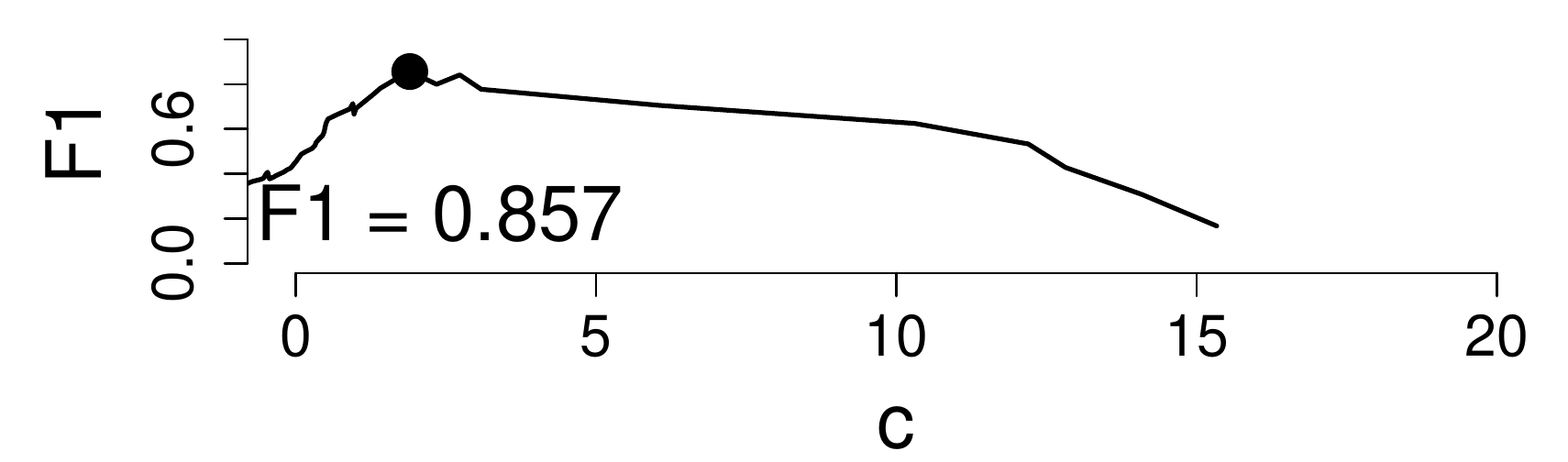}}
  \subfloat[\syntheticns(2000,5) with SVM]{\includegraphics[width = 0.33\textwidth, trim=1mm 5mm 3mm 20mm, clip]{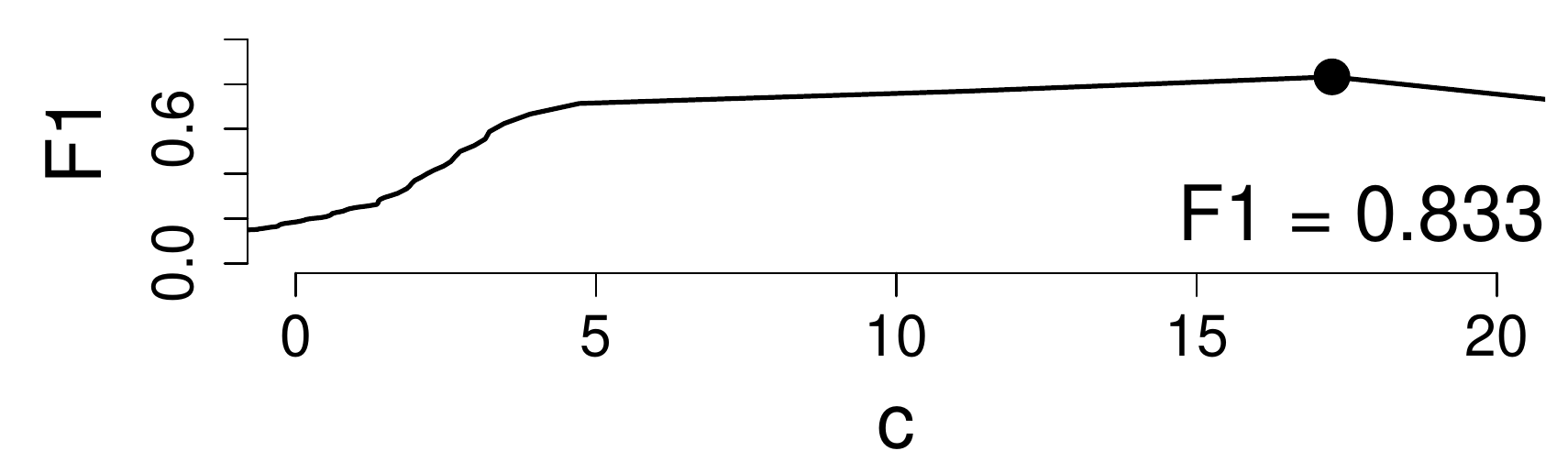}}
  \subfloat[\syntheticns(2000,5) with RF]{\includegraphics[width = 0.33\textwidth,  trim=1mm 5mm 3mm 20mm, clip]{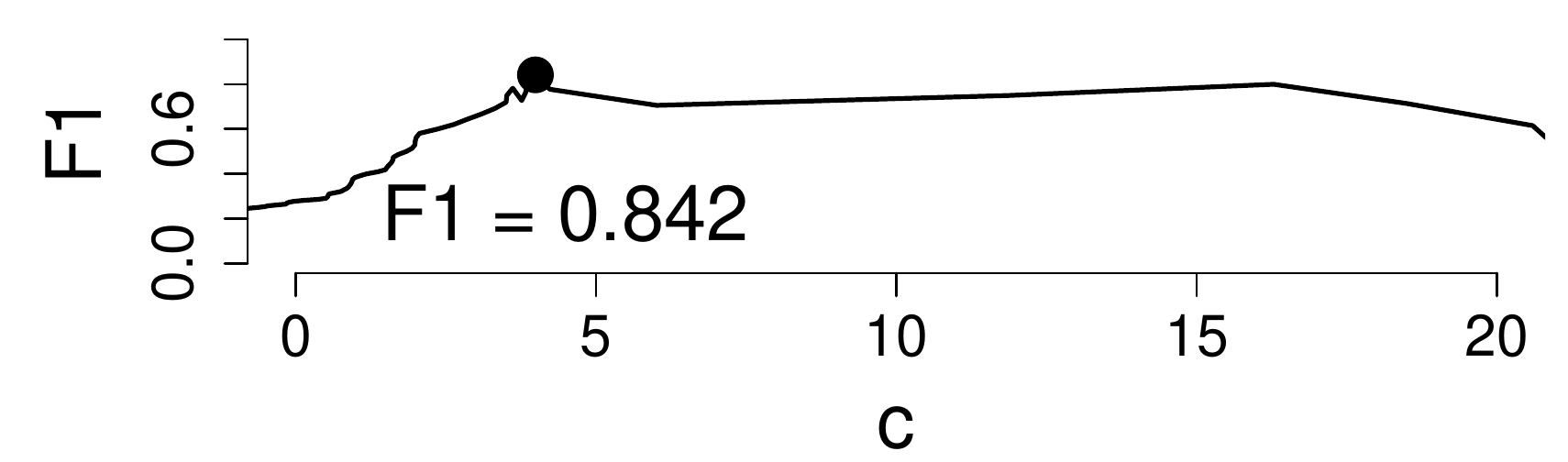}}\\
     \subfloat[\airquality]{\includegraphics[width = 0.33\textwidth, trim=1mm 5mm 3mm 20mm, clip]{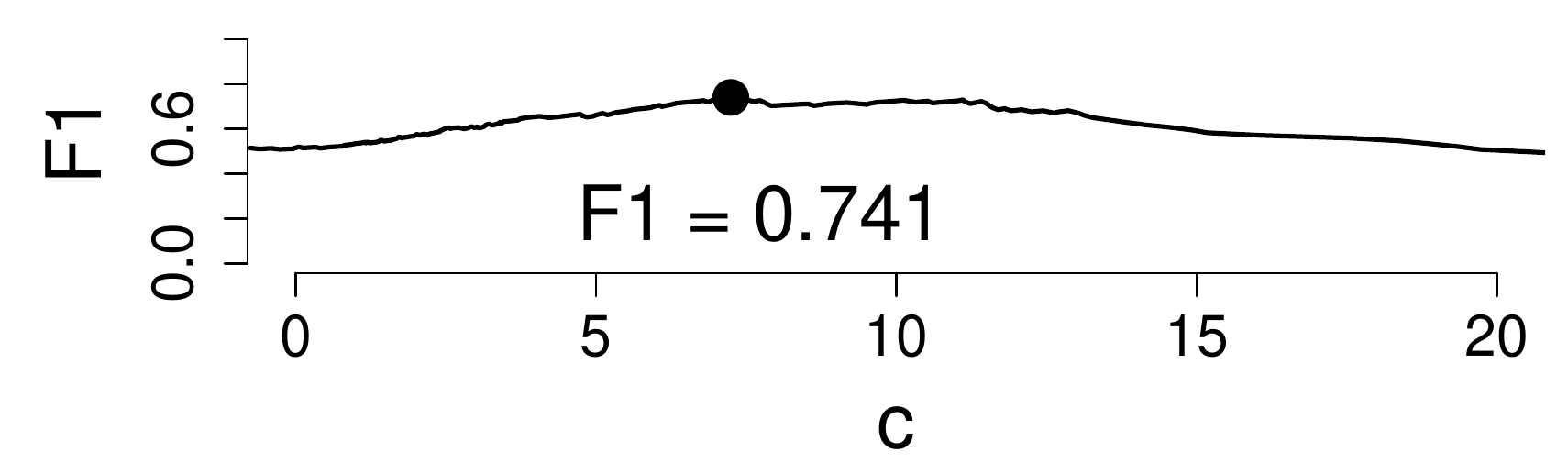}}
   \subfloat[\airline]{\includegraphics[width = 0.33\textwidth, trim=1mm 5mm 3mm 20mm, clip]{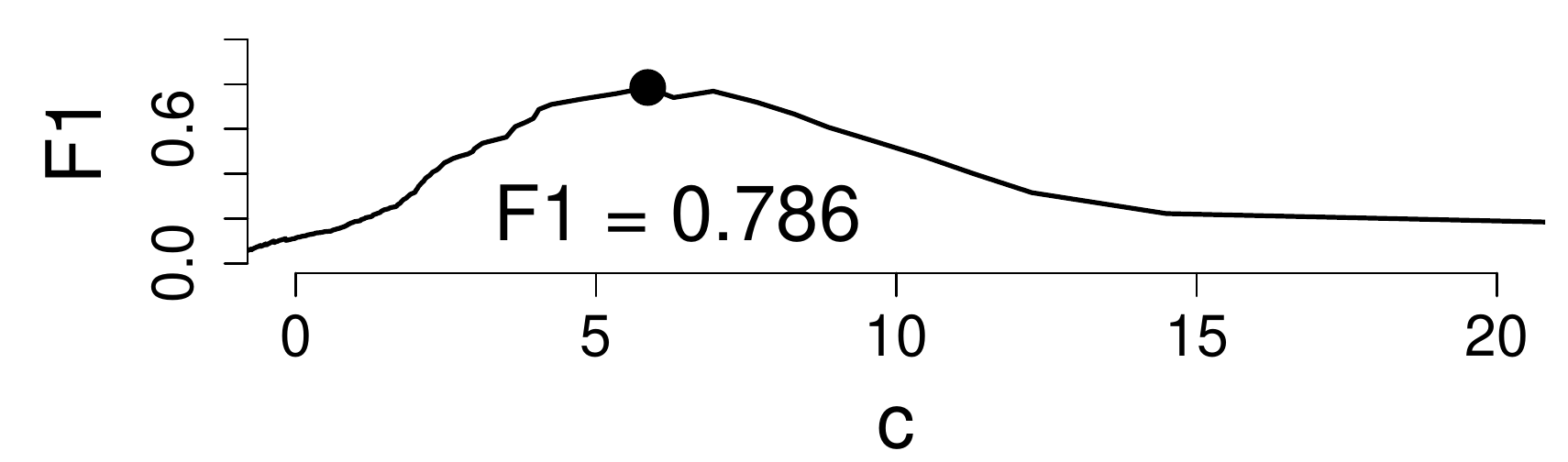}}
  \caption{The effect of the multiplier constant $c$ for $\delta$ in Eq.~\eqref{eq:delta},
    $c_\mathrm{opt}$ from Tab.~\ref{tab:k_results} shown with the corresponding $\mathit{F1}$-score.
  \label{fig:best_cs}}
\end{figure}

\subsection{Selecting a suitable drift detection threshold}\label{ssec:sel-c}
In this section we investigate how varying the value of the drift detection threshold $c$ affects the performance of the \drifter method.
For each dataset, i.e., \airquality, \airline, and \syntheticns(2000,5), we used the fixed parameter values
as defined in Sec. \ref{ssec:parameters} and selected the concept length (using the parameter $k$) based on the previous experiment, i.e.,
we used the $k$ resulting in the best performance in terms of the $\mathit{F1}$-score using the optimal $c$ (shown in bold in Tab.~\ref{tab:k_results}). We excluded the \bikesharingr and \bikesharingd datasets here, since they do not contain virtual concept drift, which makes the
relation between the $\mathit{F1}$-score and $c$ less informative.
The results are presented in Fig.~\ref{fig:best_cs}. We conclude that the
performance of our method is quite insensitive to the value of $c$
and that a
threshold value of $c=5$ seems to be a robust choice for the datasets considered.

\begin{figure}
  \centering
  \subfloat[Varying $n$.]{\includegraphics[height=3.15cm]{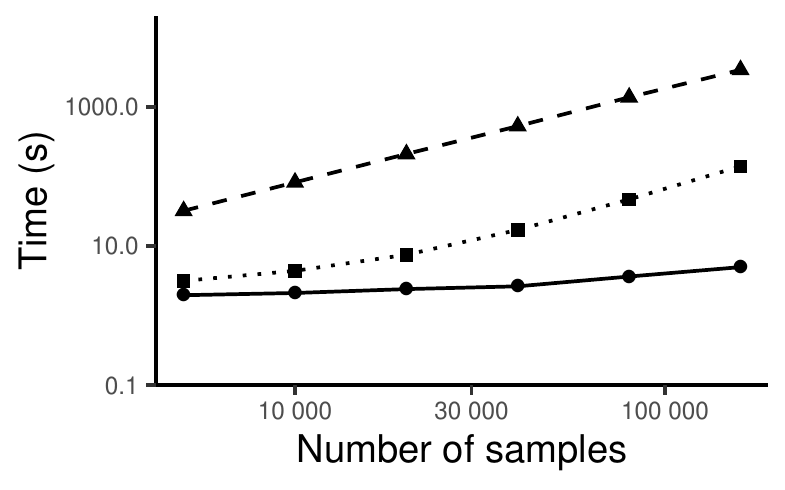}}
  \subfloat[Varying $m$.]{\includegraphics[height=3.15cm]{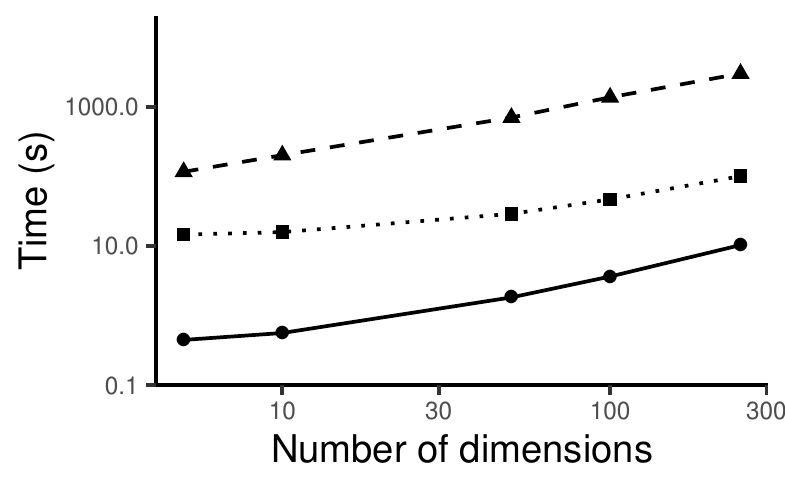}}
  \subfloat[Varying $k$.]{\includegraphics[height=3.15cm]{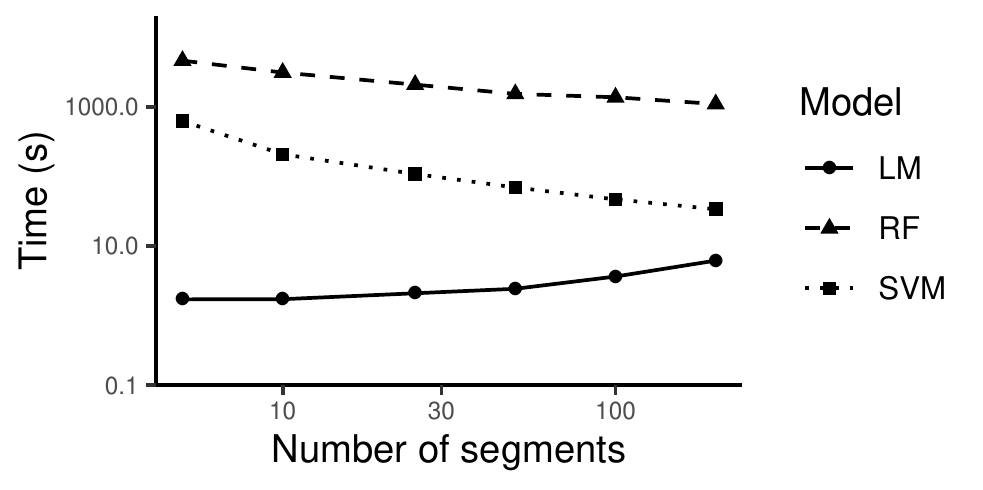}}
  \caption{Scalability of the \drifter algorithm in the training phase using \syntheticns($n$,$m$).
    In each figure, one of the parameters $n$ (training data length), $m$ (data dimension), and $k$ (number of training segments) is varied,  while the remaining ones are kept constant ($n=80\,000$, $m=100$, and $k=100$). 
  }
    \label{fig:res:scalability:tr}
\end{figure}

\begin{figure}
  \centering
  \subfloat[Varying $n$.]{\includegraphics[height=3.15cm]{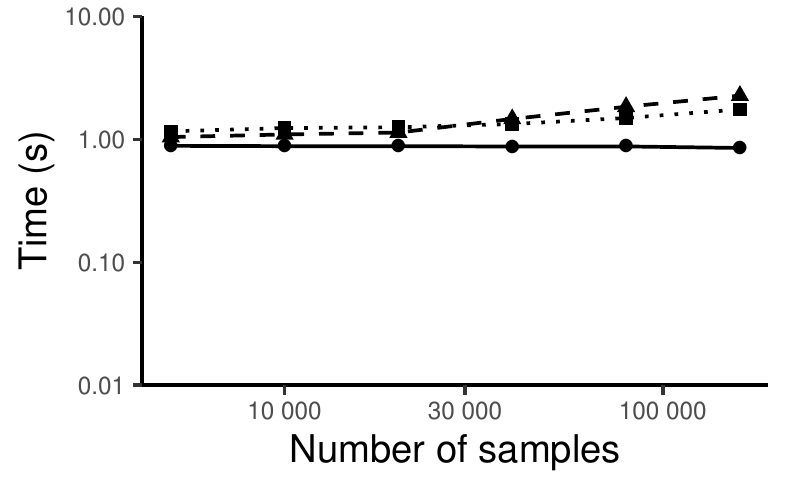}}
  \subfloat[Varying $m$.]{\includegraphics[height=3.15cm]{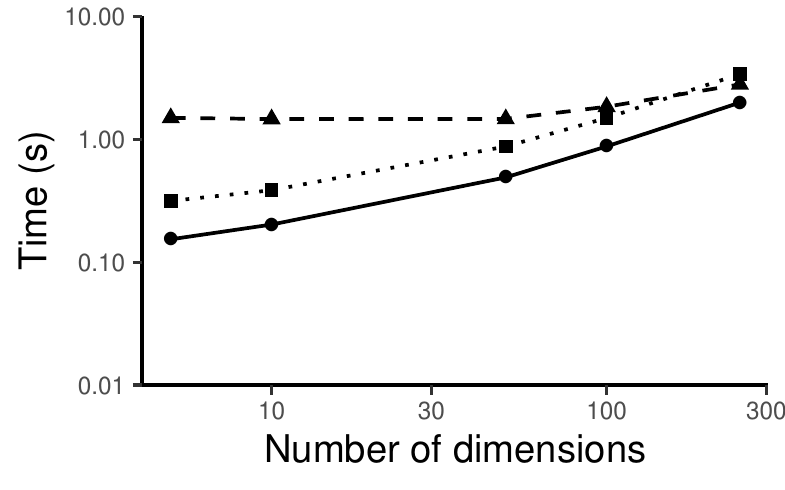}}
  \subfloat[Varying $k$.]{\includegraphics[height=3.15cm]{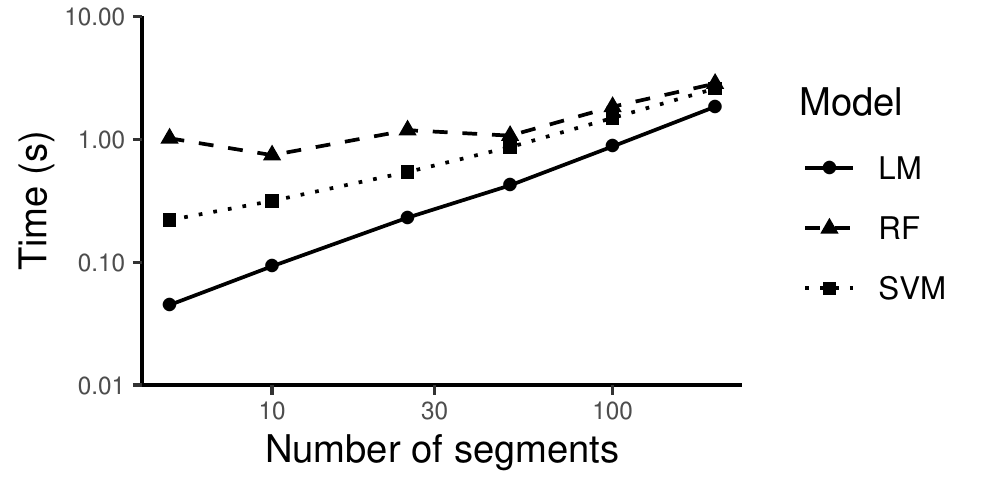}}
  \caption{Scalability of the \drifter algorithm in the testing phase with \syntheticns($15$,$m$), when trained using  \syntheticns($n$,$m$). 
  In each figure, one of the parameters $n$ (training data length), $m$ (data dimension), and $k$  (number of training segments) is varied,  while the remaining ones are kept constant ($n=80\,000$, $m=100$, and $k=100$). 
  }
    \label{fig:res:scalability:te}
\end{figure}

\subsection{Scalability}\label{ssec:scalability}
The scalability experiments were performed using the \syntheticns($n$,$m$) data. 
We constructed the datasets as described in Sec.~\ref{ssec:datasets}, and varied the data dimensionality $m\in\{5, 10, 50, 100, 250\}$, the length $n\in\{5\,000, 10\,000, 20\,000, 40\,000, 80\,000, 160\,000\}$ of training data, and
the parameter $k\in\{5, 10, 25, 50, 100, 200\}$ controlling the number of segments (and hence, $l_\mathrm{tr}$). 
We used \syntheticns($n$,$m$)
as the training dataset  and generated a testing dataset of constant length $l_\mathrm{te}=15$, i.e., we used \syntheticns($15$,$m$) as the testing data. 
Additionally, we used 
 $amp = 5$ for the testing data generation, while  $amp = 1$ was used for the training data. 
Since the actual training of the full model $f$ is not part of \drifter and the quality of the model $f$ is not relevant here, 
we used the first 500 samples of the training data to train an SVM regressor $f$. 
We then varied the choice for the model family (LM, SVM, RF) used by the \drifter in training the segment models.  

The median running time of the training and testing phases of \drifter over five runs are shown in Fig.~\ref{fig:res:scalability:tr} and
Fig.~\ref{fig:res:scalability:te}, respectively.
Here we observe that indeed the training phase is the dominant factor affecting the scalability, as discussed in Sec. \ref{ssec:fulldrifter}, and in particular when using OLS regression to train the segment models, our drift detection algorithm works fast for reasonable dataset sizes.

\begin{figure}
  \centering
    \subfloat[\synthetic (LM)]{\includegraphics[width = \textwidth, trim=27mm 25mm 15mm 20mm, clip]{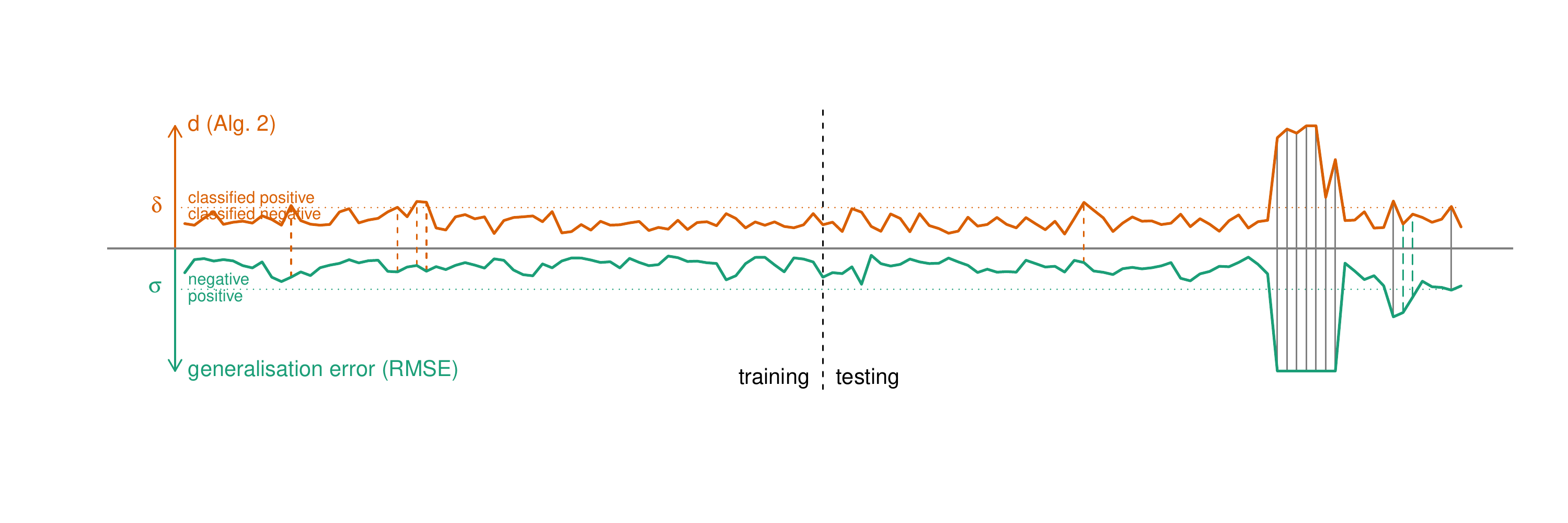}}\\
        \subfloat[\synthetic (SVM)]{\includegraphics[width = \textwidth, trim=27mm 25mm 15mm 20mm, clip]{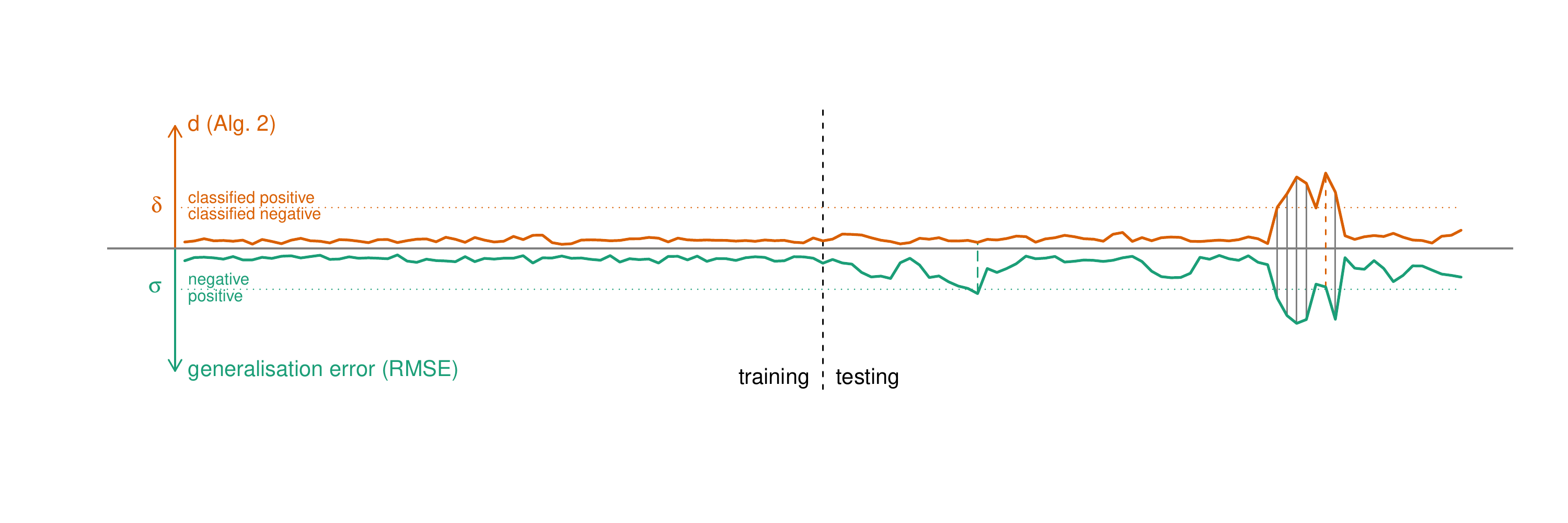}}\\
   \subfloat[\synthetic (RF)]{\includegraphics[width = \textwidth, trim=27mm 25mm 15mm 20mm, clip]{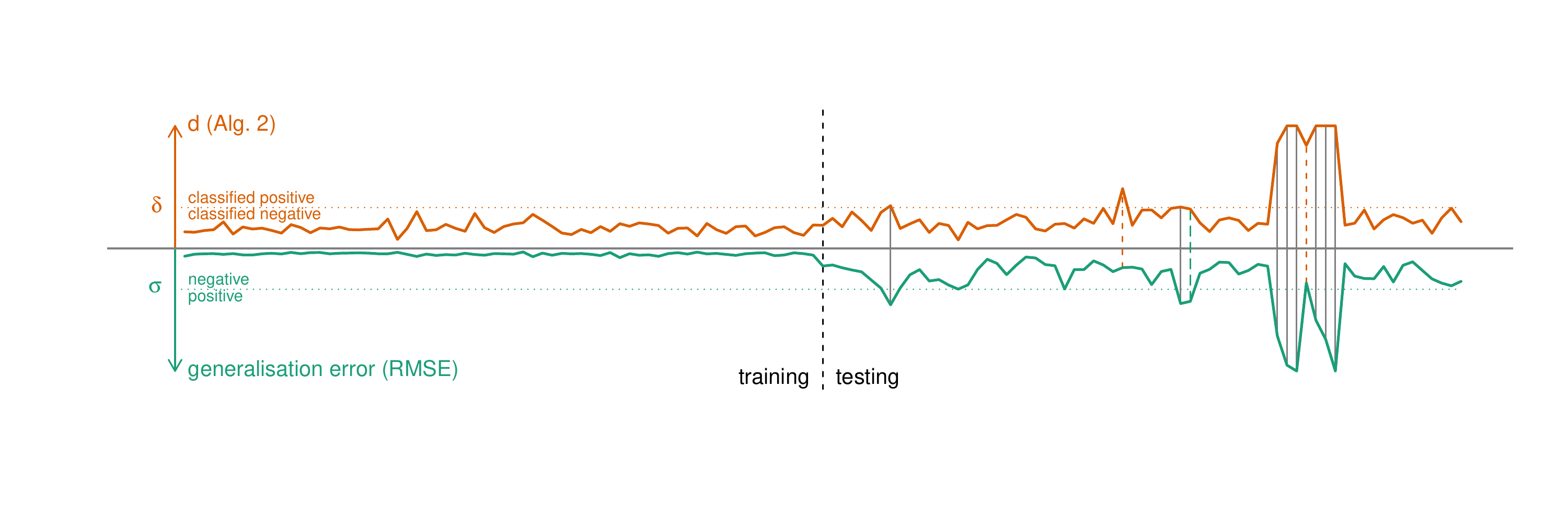}}\\
  \caption{The generalization error and concept drift indicator $d$ for the test segments
  of length $l_\mathrm{te}=15$ in the \syntheticns(2000,5) dataset. Here, $\delta$ denotes
     the concept drift detection threshold and $\sigma$ denotes the generalization error threshold. 
    The vertical lines between the two curves
    indicate the segments that are true positives (gray solid line), false positives (orange dashed line), and false negatives (green longdash line).
    }
    \label{fig:res:plots1}
\end{figure}

\begin{figure}
  \centering
   \subfloat[\airquality]{\includegraphics[width = \textwidth, trim=27mm 25mm 18mm 20mm, clip]{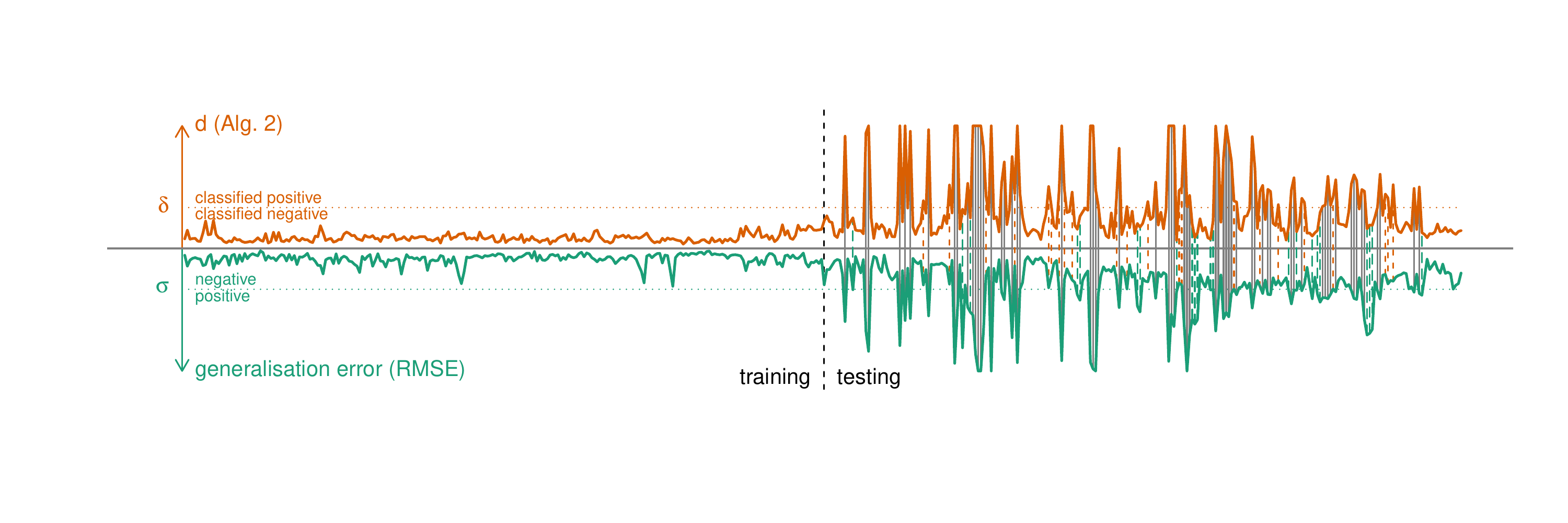}}\\
  \subfloat[\airline]{\includegraphics[width = \textwidth, trim=27mm 25mm 18mm 20mm, clip]{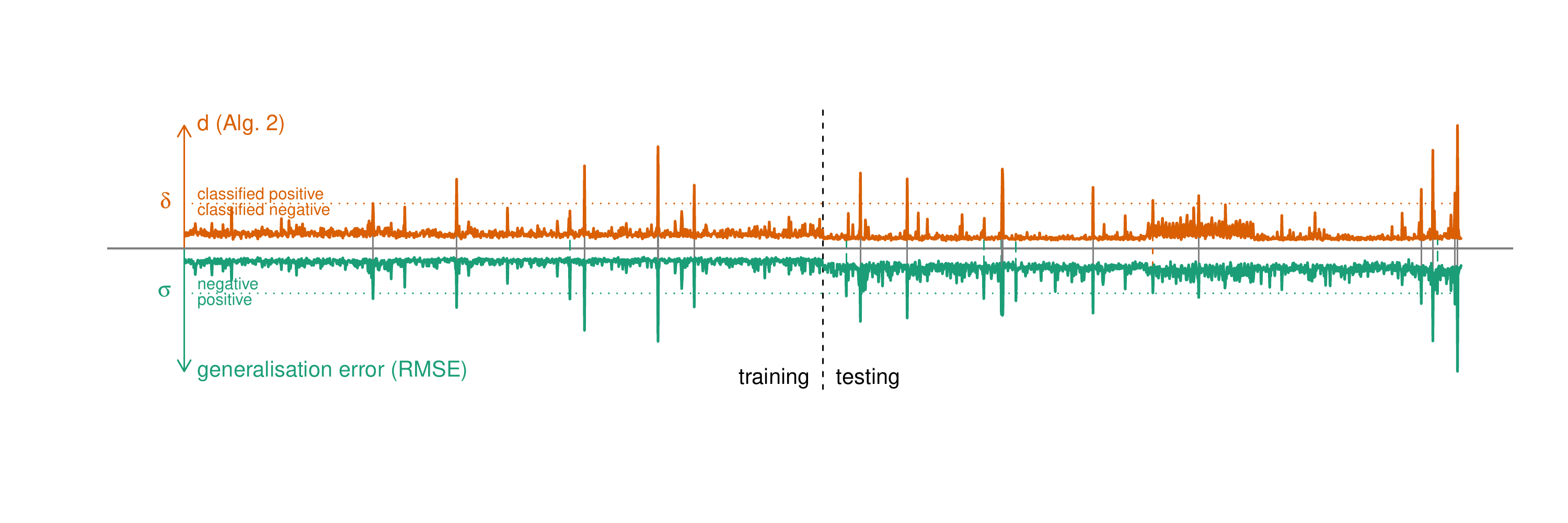}}\\
    \subfloat[\bikesharingr]{\includegraphics[width = \textwidth, trim=22mm 25mm 15mm 20mm, clip]{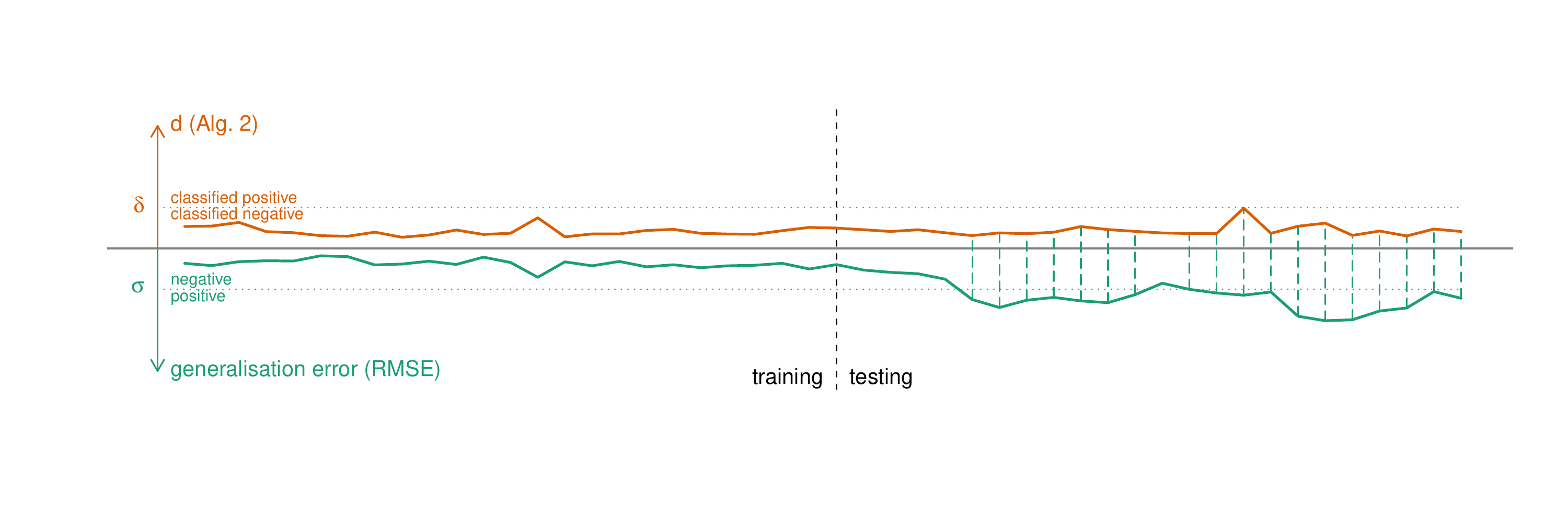}}\\
  \subfloat[\bikesharingd]{\includegraphics[width = \textwidth, trim=22mm 25mm 15mm 20mm, clip]{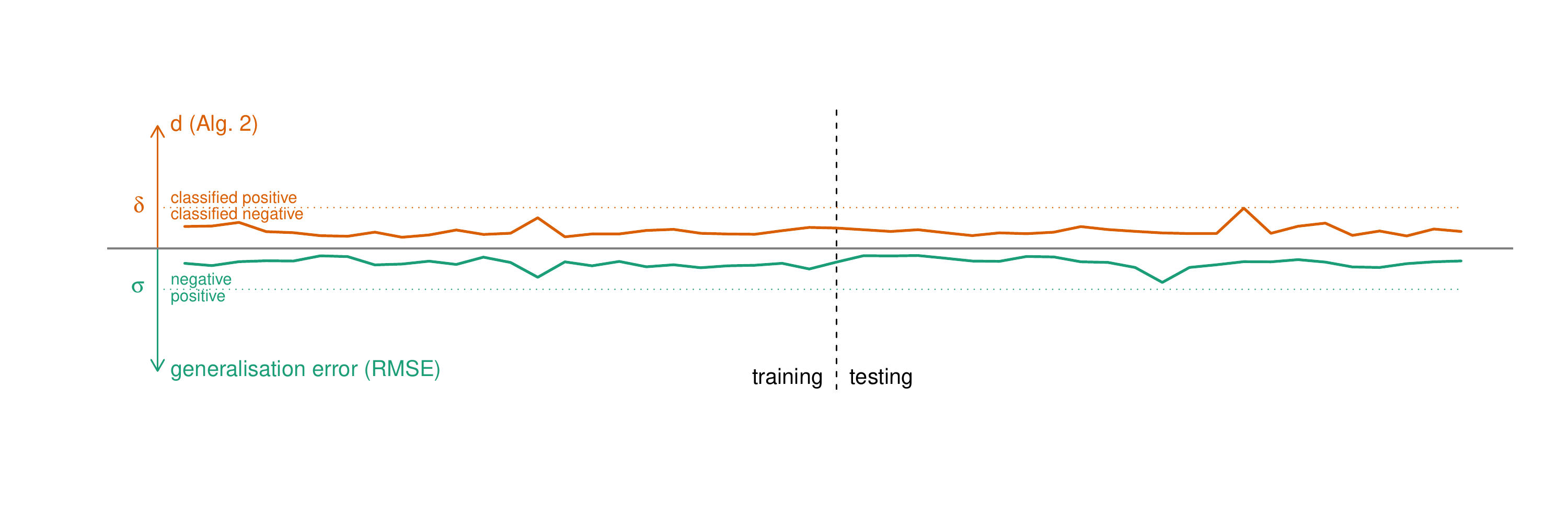}}
  \caption{
  The generalization error and concept drift indicator $d$ for  test segments
  of length $l_\mathrm{te}=15$  in \airquality, \airline, and \bikesharing datasets. 
  Here, $\delta$ denotes the concept drift detection threshold and $\sigma$ denotes the generalization error threshold. 
    The vertical lines between the two curves
    indicate the segments that are true positives (gray solid line), false positives (orange dashed line), and false negatives (green longdash line).
  }
    \label{fig:res:plots2}
\end{figure}

\subsection{Detection of concept drift}\label{ssec:detection}

Finally, we consider examples illustrating how our method for detecting concept drift works in
practice. In Figures~\ref{fig:res:plots1} and~\ref{fig:res:plots2}, we show the generalization error (green lines) and the concept drift
 indicator value $d$ (orange line), 
and in Fig.~\ref{fig:res:roc} the ROC-curves for \syntheticns(2000,5), \airquality, \airline, \bikesharingr, and \bikesharingd datasets. 
For \syntheticns(2000,5), \airline, and \airquality data we have used parameters bolded in Tab.~\ref{tab:k_results}, and for \bikesharingr, and \bikesharingd we used $k=4$ and selected $\delta$ to be larger than maximal value of $d$, 
because for \bikesharingr 
a negative value of $c$ would lead to a non-sensical value of $\delta$, 
and because \bikesharingd does not have concept drift at all.

For the \syntheticns(2000,5) data (Fig.~\ref{fig:res:plots1}) we observe that our algorithm can detect the virtual concept drift introduced during the period $[1700,1800]$.

\begin{figure}[t]
  \centering
    \subfloat[\synthetic (LM)]{\includegraphics[width = 0.31\textwidth,trim=1mm 2mm 1mm 3mm, clip]{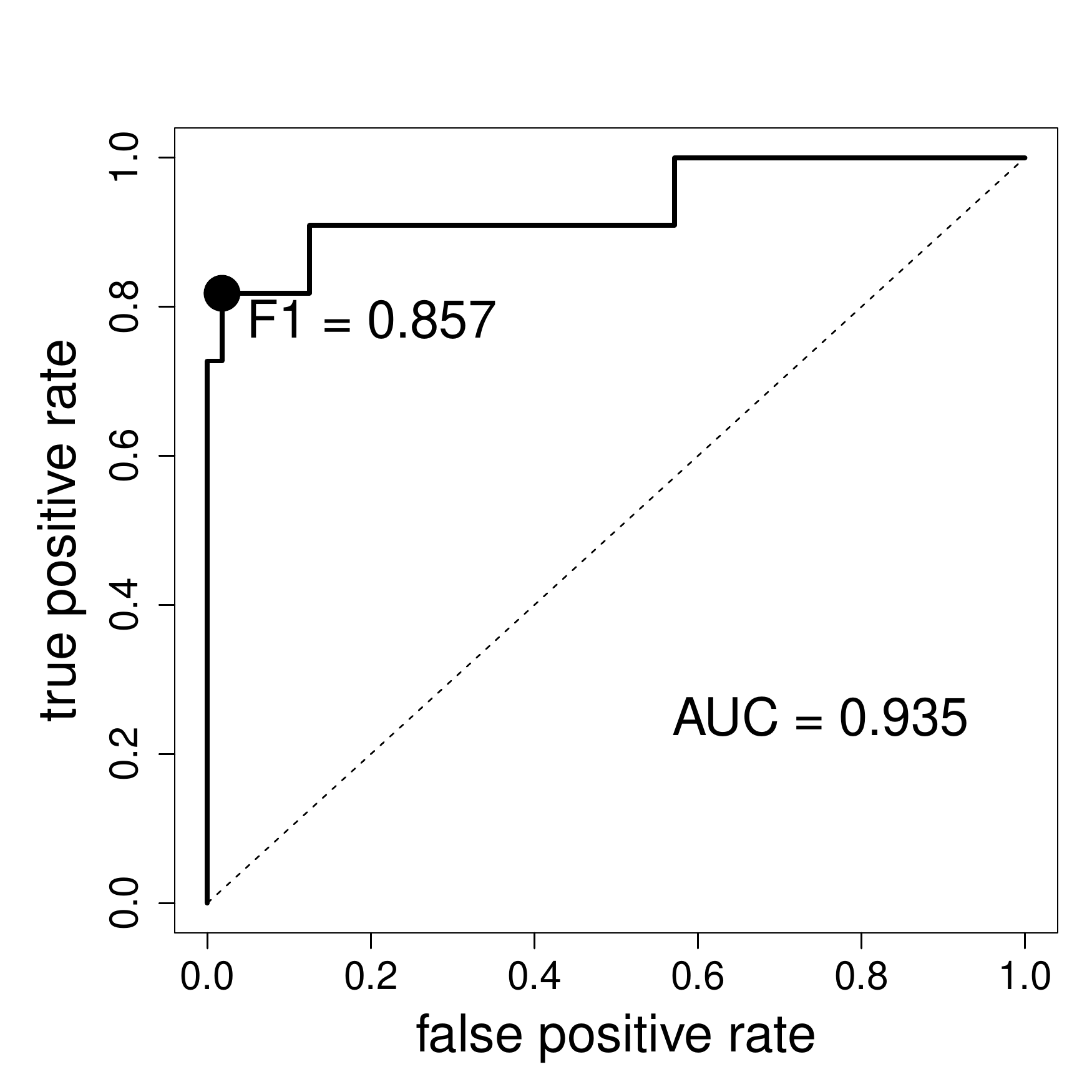}}
    \subfloat[\synthetic (SVM)]{\includegraphics[width = 0.31\textwidth,trim=1mm 2mm 1mm 3mm, clip]{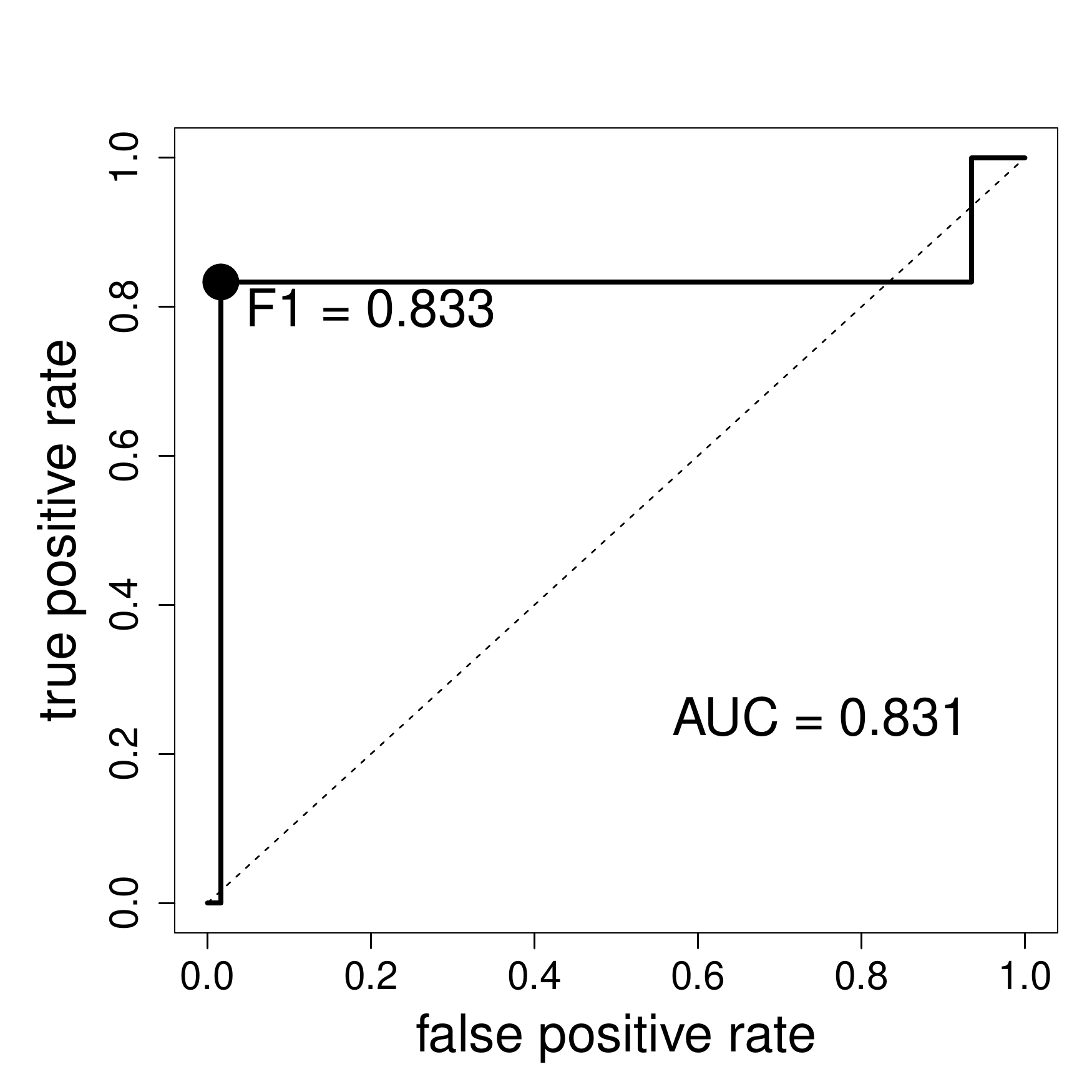}}
    \subfloat[\synthetic (RF)]{\includegraphics[width = 0.31\textwidth,trim=1mm 2mm 1mm 3mm, clip]{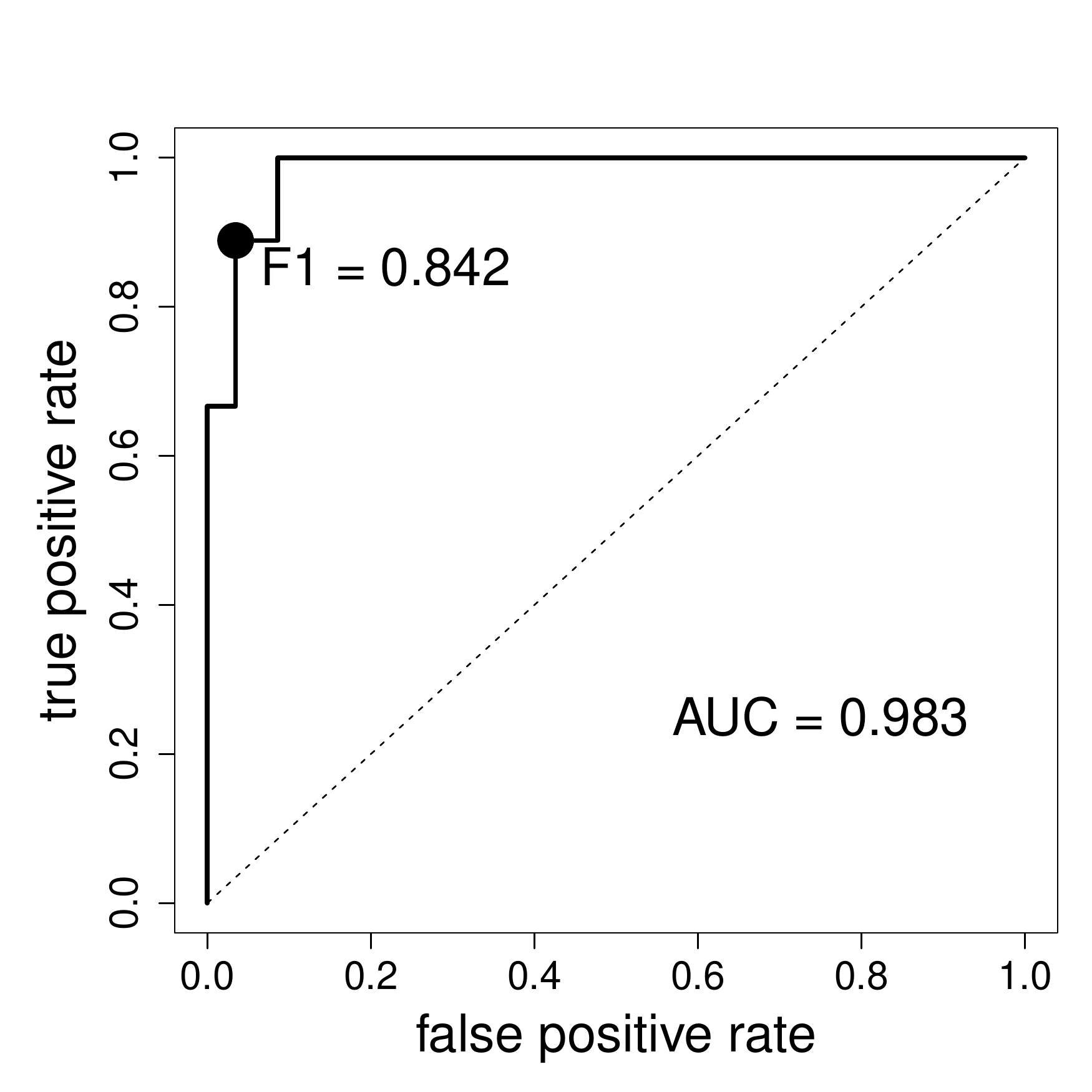}}\\
  \subfloat[\airquality]{\includegraphics[width = 0.31\textwidth, trim=1mm 2mm 1mm 3mm, clip]{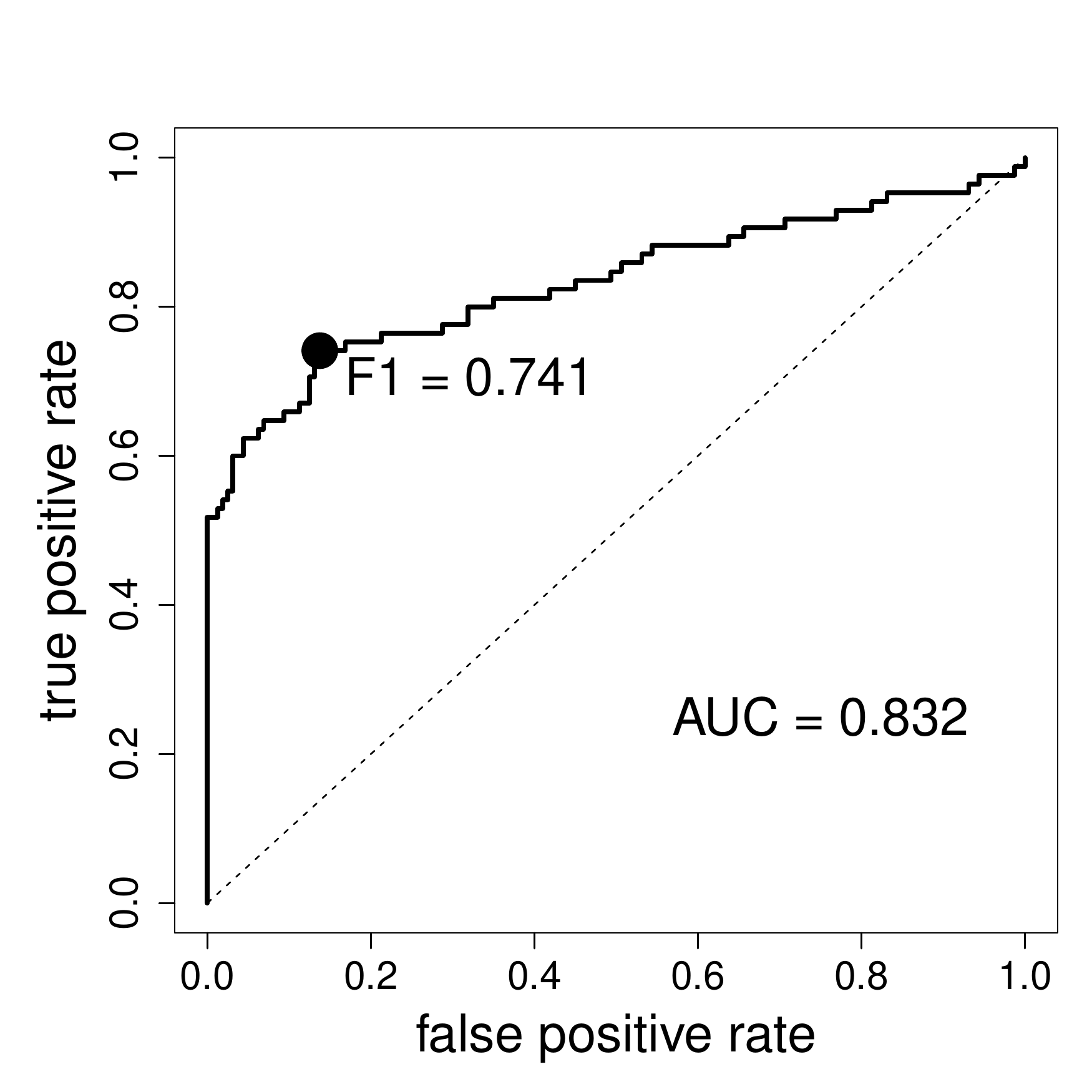}}
  \subfloat[\airline]{\includegraphics[width = 0.31\textwidth,trim=1mm 2mm 1mm 3mm, clip]{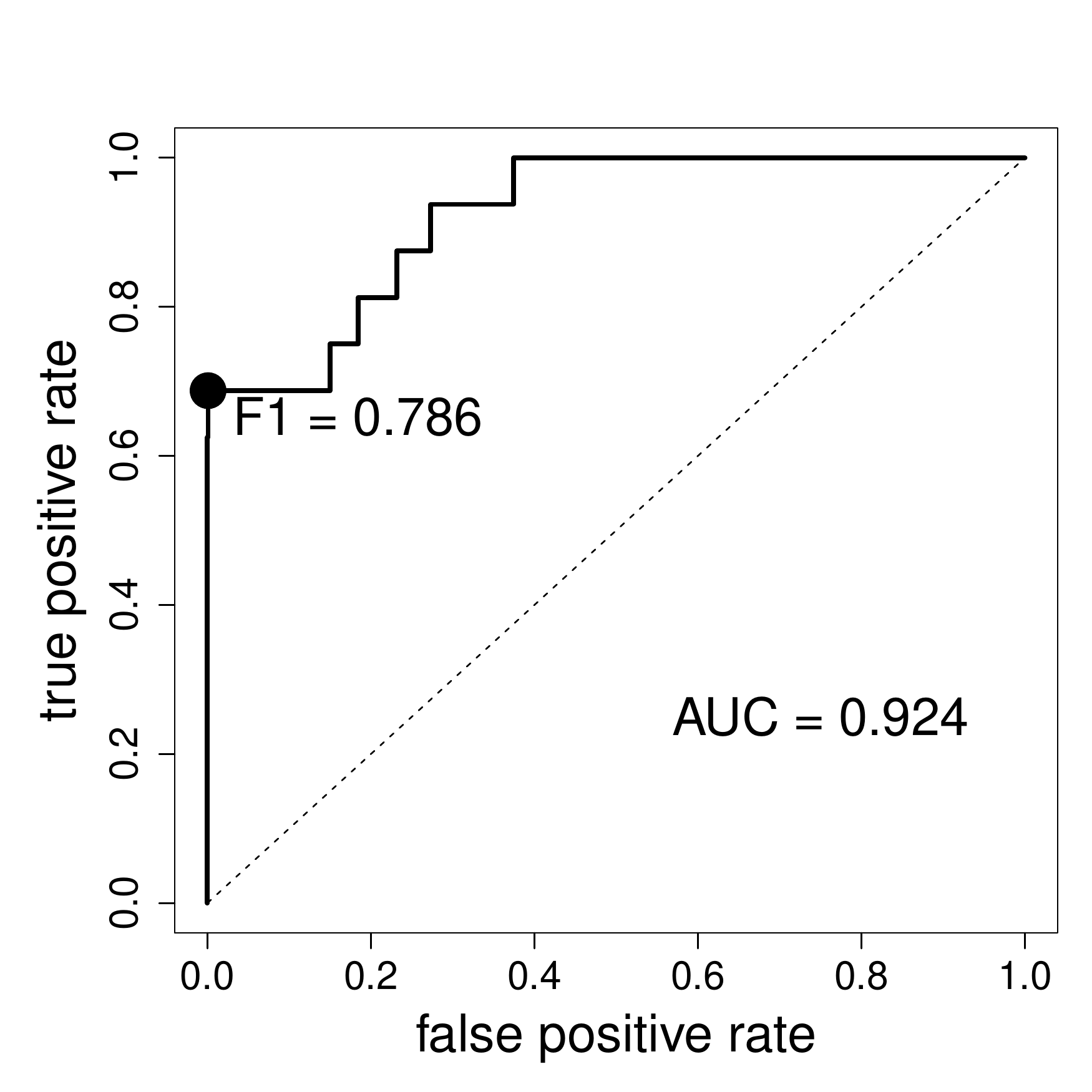}}
  \caption{ROC-curves \cite{fawcett2006} for \syntheticns(2000,5), \airquality, and \airline datasets.}
    \label{fig:res:roc}
\end{figure}

For the \airquality data (Fig.~\ref{fig:res:plots2}a) we observe that a significant amount of the testing data seems to exhibit concept drift, and our algorithm detects this. There is a rather natural explanation for this. The \airquality data contains measurements of a period of one year.
The model $f_\mathrm{AQ}$ has been trained on data covering the spring and summer months (March to August), while the testing period consists of
the autumn and winter months (September to February). Hence, it is natural that the testing data contains concepts not present in the training data.
Furthermore, one should observe that the last segments of data again begin to resemble the training data, and hence we do not observe concept drift in these segments. 

For the \airline data, we observe that that some of the segments in the training data  also have a rather high generalization value for the error, indicating that there are parts of the training data that the regressor $f_{AL}$
does not model particularly well. However, the  concept drift indicator $d$ behaves similarly to RMSE (both for segments in the training and testing data),
demonstrating that it can be used to estimate when the generalization error would be high.

For the \bikesharingr data (Fig.~\ref{fig:res:plots2}c) we observe that even though the generalization error is large for most of the segments in the testing data, the drift detection indicator does not indicate concept drift. This is explained by the \emph{real concept drift} present in data, and once we have removed it in the \bikesharingd data (Fig.~\ref{fig:res:plots2}d) we observe no concept drift.
We hence observe that a considerable number of false negatives can indicate real concept drift in the data. 
However, in order to detect this, one needs to have access to the ground truth values.


\section{Discussion} \label{sec:discussion}

In this paper, we have presented and evaluated an efficient method for detecting
concept drift in regression models when the ground truth is unknown.
In this paper, we define concept drift as a phenomenon causing larger than
expected estimation errors on new data, as a result of changes in the generating
distribution of the data. Defining concept drift in terms of the estimation
error, instead of considering all changes in the distribution, makes it
possible to detect only the changes that actually affect the prediction
quality. Thus, if concept drift detection is used to monitor the performance of a
regression model, it reduces the false positives rate.
It is surprising
how little attention this problem has received, considering its importance in
multiple domains.

When the dependent variable $y$ is unknown it is only possible to detect changes in
the distribution of the covariates $p(x)$. Our idea is to use the regression functions themselves
to study the changes in this distribution. As we have shown for linear models
in Thm. \ref{thm:monotonic}, we postulate that if we train two or more regression functions on
different subsets of the data, then the difference in the estimates given by the regression functions
contains information about the generalization error. This method, while simple, is powerful.
It, e.g.,  ignores by design features of the data that are irrelevant for estimating~$y$.
The underlying assumption is that by using subsets of the training data we can train regressors that can
capture {\em concepts} in the data, and if the testing data contains concepts not found
in the training data, then it is likely that there is concept drift.
The \drifter method presented in this paper also scales
well. Especially high performance is reached using OLS linear
segment models.

In this paper, we have used models trained using different segments of the data. 
As future work,
an interesting topic to study is how the data could be ``optimally'' partitioned for this problem.
Another alternative---which we have experimented with but not reported here---is to
train several regression models from different model families on the data. In this paper we have
also focused on estimating the generalization error of a regression function. The same ideas
could be applied to  detect concept drift in classifiers as well.

The theoretical foundation for this approach is shown to hold in the simple case of linear regression. However, our empirical evaluation with real datasets of various types (and different regressors) demonstrates that the idea also works when there are sources of non-linearity. Our experiments suggest that often the (black-box) regressor given as input can be locally approximated using linear regressors, and the differences between the estimates from these regressors serve as a good indicator for concept drift. The current paper represents initial work towards a practical concept drift detection algorithm,
with experimental evaluation illustrating parameters that work robustly for the datasets considered in this work. 
Further work is needed to establish general practices for selecting suitable parameters for the \drifter method.

\begin{acknowledgments}
We thank Dr Martha Zaidan for help and discussions.
This work was funded by the Academy of Finland (decisions 326280 and 326339).
We acknowledge the computational
resources provided
by Finnish Grid and Cloud Infrastructure
\cite{fcgi}.
\end{acknowledgments}

\bibliography{ms}

\end{document}